%% file: neurips_2026.tex
\pdfoutput=1 
\documentclass{article}

\usepackage[preprint]{neurips_2026}

\usepackage[utf8]{inputenc} 
\usepackage[T1]{fontenc}    
\usepackage{hyperref}       
\usepackage{url}            
\usepackage{booktabs}       
\usepackage{amsmath}        
\usepackage{amsfonts}       
\usepackage{amssymb}        
\usepackage{amsthm}         
\usepackage{nicefrac}       
\usepackage{microtype}      
\usepackage{xcolor}         
\usepackage{booktabs}
\usepackage{graphicx}

\usepackage{algorithm}
\usepackage{algpseudocode}

\newtheorem{theorem}{Theorem}
\newtheorem{lemma}[theorem]{Lemma}
\newtheorem{proposition}[theorem]{Proposition}
\newtheorem{corollary}[theorem]{Corollary}
\theoremstyle{definition}
\newtheorem{assumption}[theorem]{Assumption}

\theoremstyle{remark}

\providecommand{\KL}{\mathrm{KL}}
\providecommand{\E}{\mathbb{E}}
\providecommand{\R}{\mathbb{R}}
\providecommand{\N}{\mathcal{N}}

\providecommand{\Cov}{\mathrm{Cov}}
\providecommand{\tr}{\mathrm{tr}}
\providecommand{\norm}[1]{\lVert #1 \rVert}

\title{The Value of Covariance Matching in Gaussian DDPMs and the Lanczos Sampler}

\author{Md Sahil Akhtar \\
   Electrical Engineering and Computer Science \\
  Massachusetts Institute of Technology \\
  \texttt{sahil003@mit.edu} \\
  \And
  Aymane El Gadarri \\
  Operations Research Center \\
  Massachusetts Institute of Technology \\
  \texttt{aymelgad@mit.edu} \\
  \AND
  Vivek F. Farias \\
  Sloan School of Management \\
  Massachusetts Institute of Technology \\
  \texttt{vivekf@mit.edu} \\
  \And
  Adam D. Jozefiak \\
Operations Research Center \\
Massachusetts Institute of Technology \\
  \texttt{jozefiak@mit.edu} \\
}

\begin{document}

\maketitle

\begin{abstract}
A central error measure in Gaussian DDPMs is the path-space KL divergence between the exact reverse chain and the learned Gaussian reverse process. This quantity is especially relevant for procedures such as classifier guidance, which perturb the entire reverse trajectory rather than only the terminal sample. Prior analyses show that standard isotropic reverse covariances suffer an unavoidable $\Omega(1/T)$ path-KL error as the number of denoising steps $T$ grows. We show that matching the full posterior covariance breaks this barrier, yielding an order-wise improvement that reduces the path KL to $O(1/T^2)$. To make full covariance matching practical, we introduce the Lanczos Gaussian sampler (LGS), a training-free, matrix-free method for sampling from the optimal reverse covariance using only covariance-vector products, which are available through Jacobian-vector products of the posterior mean. LGS avoids dense covariance storage and auxiliary covariance models. We prove that LGS approximation error decays exponentially in the number of Lanczos steps, where each Lanczos step requires a single Jacobian-vector product. Empirically, using only just three such steps improves sample quality over strong diagonal-covariance baselines, including OCM-DDPM, across standard image benchmarks. This identifies full covariance matching as both theoretically valuable and practically accessible for fast DDPM sampling.
\end{abstract}

\input{sections/intro.tex}

\input{sections/setup.tex}
\input{sections/main_result.tex}
\input{sections/lanczos.tex}

\input{sections/experiments.tex}
\input{sections/conclusion.tex}

\bibliographystyle{plainnat}
{\small\bibliography{references}}

\appendix

\input{sections/appendix.tex}

\input{sections/appendix/appendix_counter_example.tex}
\input{sections/appendix/appendix_lanczos.tex}

\input{sections/appendix/appendix_main_result.tex}
\input{sections/appendix/appendix_experiments.tex}


\end{document}

%% file: sections/intro.tex
\section{Introduction}

A central approximation error in the study of Gaussian Denoising Diffusion Probabilistic Models (DDPMs) is the path-space KL divergence between the exact reverse-time denoising chain and the learned Gaussian reverse process. Many uses of diffusion models require an accurate reverse trajectory as opposed to simply accurate endpoint samples. Classifier guidance is a canonical example, since it modifies the reverse chain at each denoising step. The unconditional path KL controls the distributional error of such classifier-guided sampling.
%

Prior analyses show that for standard Gaussian DDPM parameterizations, where the posterior (denoising) kernel is chosen to have isotropic covariance, the path KL generally cannot decay faster than $\Omega(1/T)$ as the number of denoising steps $T$ grows. We see here that this limitation persists even when the reverse process is allowed to use the best diagonal approximation to the true posterior covariance at each denoising step. 

\textbf{Value of full covariance matching: } Our first key result is that denoising with the optimal posterior covariance breaks the $1/T$ barrier. By matching the conditional covariance of the true reverse kernel, the path KL can be reduced to order $O(1/T^2)$. {\em In particular, denoising with the optimal covariance yields an order-wise improvement in the iteration complexity of Gaussian DDPM sampling.}

The above value notwithstanding, matching the optimal posterior covariance is known to be challenging in practice for at least two reasons. First, it has been shown that learning the optimal posterior covariance can make training unstable and degrade sample quality relative to simply using a scaled identity covariance in the denoising kernel. Second, even assuming one could learn the optimal covariance, storage requirements are prohibitive in practical (high) dimensions. Two important pieces of progress have been made in this regard: First, it has been observed that in Gaussian DDPMs Tweedie-type identities express the posterior covariance via the Jacobian of the posterior mean. Thus, covariance-vector products can be computed by Jacobian-vector products without explicitly storing the covariance. Second, to combat storage issues, the diagonal of the optimal covariance can be extracted without explicitly storing the covariance through the use of the Hutchinson trace estimator. 

\textbf{Algorithmic realization via matrix-free Gaussian sampling: } As opposed to estimating the covariance or an approximation thereof, we focus directly on sampling from a Gaussian whose covariance is only available through matrix-vector products.
To do this, we introduce the {\em Lanczos Gaussian Sampler (LGS)}. Our approach {\em does not} rely on sampled approximations to the covariance (ala Hutchinson) and provably only requires a small (fixed in problem dimension) number of so-called Lanczos iterations. As such, our method employs the {\em whole covariance matrix} (as opposed to a diagonal approximation) and {\em does not require learning an auxiliary model}. The approach improves upon a SOTA alternative (OCM-DDPM).  

\subsection{Related Literature} 

\paragraph{Covariance parameterizations in diffusion models.}
Diffusion models were originally introduced with Gaussian reverse kernels whose means and
covariances could both be learned~\citep{sohl2015deep}. Modern DDPMs typically simplify the
reverse covariance to a state-independent scaled identity, such as $\beta_t I$ or
$\widetilde{\beta}_t I$~\citep{ho2020denoising}, for computational stability and scalability. Several
works have revisited this simplification. \citet{nichol2021improved} learn diagonal log-variances
interpolating between $\beta_t$ and $\widetilde{\beta}_t$, while \citet{bao2022analytic} derive a
training-free KL-optimal isotropic variance for a pretrained score model. More closely related,
\citet{bao2022estimating} derive optimal covariance formulas, including full-covariance expressions,
and \citet{ou2025ocm} propose OCM-DDPM, which learns the optimal diagonal posterior covariance
using Tweedie-type identities and Hutchinson/Rademacher-style diagonal estimators. These works
show that covariance choices can substantially improve fast sampling, but their scalable
implementations remain scalar or diagonal and therefore do not exploit the full posterior covariance.

\paragraph{Beyond diagonal covariance and matrix-free sampling.}
Dense state-dependent covariance modeling is infeasible at image scale: storing a covariance costs
$O(d^2)$ memory and direct Gaussian sampling requires dense factorization. Existing non-diagonal
methods therefore impose additional structure. For example, \citet{kdct} use a Kronecker-DCT
covariance model tailored to image data, exploiting color and spatial-frequency correlations. In
contrast, our approach does not learn, store, or structurally parameterize the covariance. We use the
full posterior covariance only through covariance-vector products, available via Jacobian-vector
products of the posterior mean, and sample from the corresponding Gaussian with a matrix-free
Lanczos approximation to $\Sigma^{1/2}z$. Lanczos and Krylov methods are classical tools for
approximating matrix functions applied to vectors~\citep{cullum2002lanczos,saad2011numerical};
we exploit the approximation framework of \citet{musco2018stability} with $f(x)=\sqrt{x}$ and derive novel dimension-free exponential decay bounds of the sampling error. Unlike
Hutchinson-style estimators, which recover traces or diagonals, Lanczos uses the full covariance
implicitly and requires no auxiliary covariance model.

\paragraph{Path-KL theory and Gaussian-channel analysis.}
A large body of theory studies diffusion samplers by controlling reverse-process error, often through
path-space or Girsanov arguments~\citep{chen2022sampling,benton2023nearly,lee2023convergence,
li2024towards,conforti2025kl}. For DDPMs,
\citet{chen2022sampling} prove an $O(d/T)$ path-KL upper bound and give examples showing that
this rate is tight for standard covariance choices. Recent works obtain faster
$\widetilde O(1/T^2)$ rates for endpoint KL or terminal marginal error under fixed isotropic
covariance~\citep{jiao2025optimal,jain2025sharp}. These results, however, control only the endpoint KL and do not bound the path KL.
 This distinction is important for procedures that modify the reverse trajectory, such as classifier guidance. Indeed, in Appendix~\ref{sec:classifier_guidance} and Appendix~\ref{app:counter-example},
  we show that without matching the optimal reverse kernel covariance, the conditional endpoint KL under classifier guidance can still scale as $\Omega(1/T)$.
 Our analysis explains how full covariance matching breaks
the $\Theta(1/T)$ path-KL barrier. The key step is an information-theoretic representation of the
one-step denoising loss as a Gaussian-channel mutual information, combined with the I-MMSE
identity~\citep{guo2005mutual} and higher-order derivative formulas~\citep{payaro2009hessian,
nguyen2024derivatives}.

%% file: sections/setup.tex
\section{Setup and Results}

\begin{figure}[t]
  \centering
  \makebox[\textwidth][c]{%
  \begin{minipage}[c]{0.42\textwidth}
    \centering
    \includegraphics[width=\linewidth]{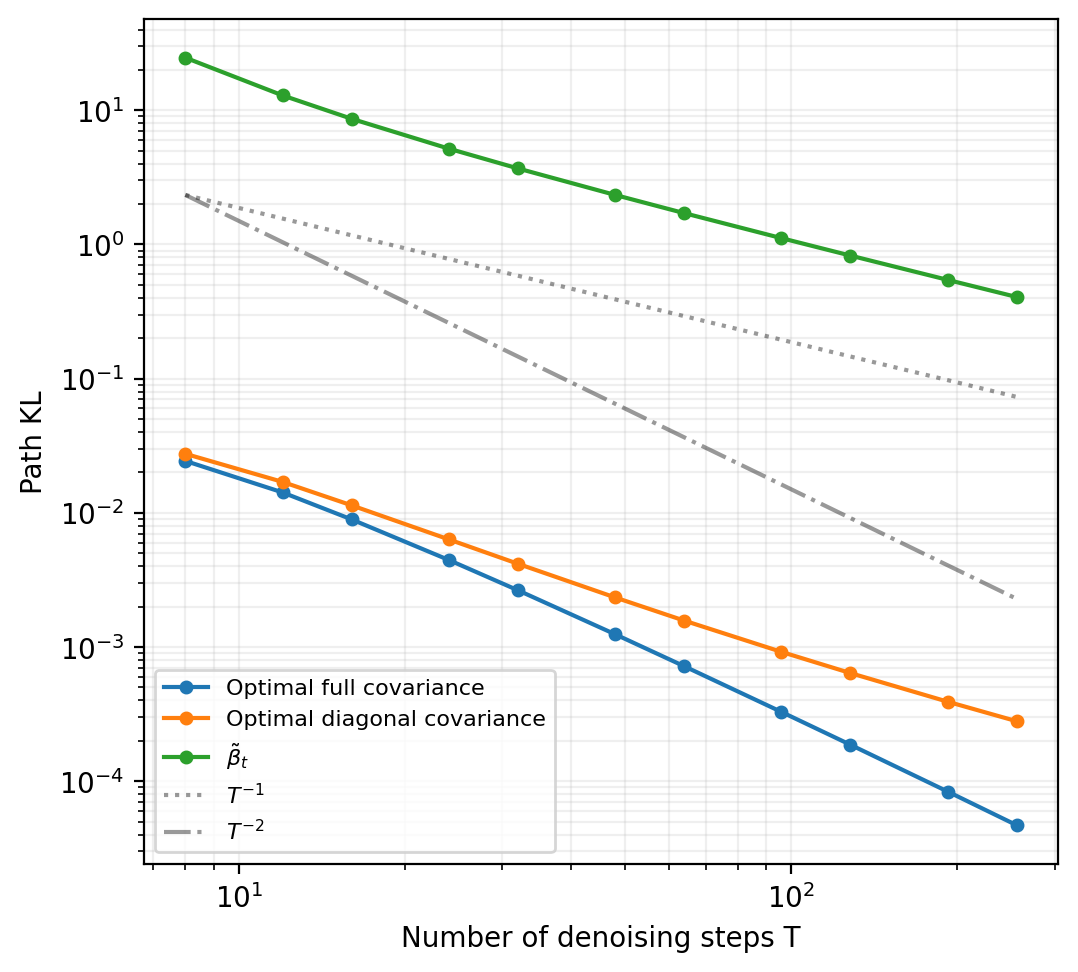}
  \end{minipage}%
  \hspace{0.02\textwidth}%
  \begin{minipage}[c]{0.40\textwidth}
    \centering
    {\footnotesize
    \setlength{\tabcolsep}{3.6pt}
    \renewcommand{\arraystretch}{0.95}
    \begin{tabular}{lcccc}
      \toprule
      & \multicolumn{4}{c}{CELEBA 64$\times$64} \\
      \cmidrule(lr){2-5}
      \# timesteps $K$ & 25 & 50 & 100 & 200 \\
      \midrule
      DDPM, $\beta$
      & 116.16 & 55.39 & 29.23 & 14.39 \\
      DDPM, $\tilde{\beta}$
      & 24.73 & 18.76 & 14.26 & 10.57 \\
      OCM-DDPM
      & 12.66 & 9.16 & 7.09 & 6.04 \\
      Lanczos $(m=3)$
      & 9.58 & 7.01 & 5.21 & 4.60 \\
      Lanczos $(m=5)$
      & \textbf{8.84} & \textbf{6.29} & \textbf{4.79} & \textbf{4.40} \\
      \bottomrule
    \end{tabular}
    }
  \end{minipage}%
  }
  \caption{(Left) Path KL vs. no. of denoising steps with full covariance matching (blue), $\tilde{\beta}_t$-scaled isotropic covariance (green) and diagonal covariance (orange). 
Path KL with isotropic covariance scales like $\Theta(1/T)$ in agreement lower bounds of \cite{chen2022sampling}. Allowing for a diagonal approximation of the covariance yields lower path KL but remains $\Theta(1/T)$. Matching full covariance yields an order-wise improvement, $\Theta(1/T^2)$.  Right: The results of using our approach to matching the full covariance (the Lanczos Gaussian sampler) on the CELEBA 64$\times$64 dataset. OCM DDPM \citep{ou2025ocm} is a SOTA approach.  
  }
  \label{fig:pathkl_toy_celeba_subset}
\end{figure}



A diffusion model parameterized by $\theta$ is a latent variable model $p_\theta(x_0) = \int p_\theta(x_{0:T}) \; dx_{1:T}$ where $x_{1:T}$ is a sequence of latents. Given a distribution $q(x_0)$ supported on the ball $\{x: x \in \mathbb{R}^d, \|x\|\leq D\}$, $\theta$ is canonically learned by minimizing the loss
\begin{equation}
\label{eqn:diffusionLoss}
\mathcal{L}(\theta)
\triangleq
\sum_{t=2}^T
\E_{q(x_t,x_0)}\!\left[
  \KL\!\left(q(x_{t-1}\mid x_t,x_0)\,\middle\|\,p_\theta(x_{t-1}\mid x_t)\right)
\right]
\end{equation}

In the context of Gaussian DDPMs~\citep{ho2020denoising}, the sequence of latents $x_{1:T}$ is generated according to 
\begin{equation}
\label{eqn:forwardKernel}
q(x_t \mid x_{t-1})
  := \N\!\big(\sqrt{1-\beta_t}\,x_{t-1},\; \beta_t I_d\big),
  \qquad t = 1,\ldots,T, 
\end{equation}
where $\{\beta_t\}_{t=1}^{T}\subset(0,1)$ is a pre-defined step-size schedule. The diffusion model is itself parameterized as a Gaussian process: 
\begin{equation}\label{eqn:reverseKernel}
p_\theta(x_{t-1}\mid x_t)
  := \N\!\big(\mu_\theta(x_t,t),\; \Sigma_\theta(x_t,t)\big),
   \qquad t = 2,\ldots,T, 
\end{equation}
with $p(x_T) := \N(0,I)$. 
We focus on the \emph{unconditional path KL} between the target path distribution $q(x_{1:T})$ and that of the learned reverse process $p_{\theta}(x_{1:T})$, over paths stopped at $t=1$, $x_{1:T}$. \begin{equation}\label{eqn:unconditionalPathKL}
    \KL\!\left(q(x_{1:T}) \,\middle\|\, p_\theta(x_{1:T})\right).
\end{equation}
In addition to being a focal quantity of interest in analysis of diffusion modeling, \citep{chen2022sampling,benton2023nearly}, the unconditional path KL is particularly important in the context of algorithms that rely on faithful reconstruction of the entire path, e.g. classifier guided diffusion modeling~\citep{nichol2021improved}. 

\subsection{The Value of Learning the Posterior Covariance}

Let $\mu^*(\cdot,\cdot)$ and $\Sigma^*(\cdot,\cdot)$ minimize the loss~\eqref{eqn:diffusionLoss} over all functions from $\mathbb{R}^d \times \mathbb{N}$, to $\mathbb{R}^d$ and  $\mathbb{S}^d_+$ respectively. We make the following assumption on the expressivity of the class of functions $\mu_\theta$ and $\Sigma_\theta$ in our analysis: 
\begin{assumption}
There exists a $\theta^*$ such that $\mu_{\theta^*}(\cdot,\cdot) = \mu^*(\cdot,\cdot)$ and $\Sigma_{\theta^*}(\cdot,\cdot) = \Sigma^*(\cdot,\cdot)$.
\end{assumption} 

Then our main result, presented formally as Theorem~\ref{thm:mainTheorem} in Section~\ref{sec:approximation_error} shows that for an appropriate step size schedule, 
\[
\KL\!\left(q(x_{1:T}) \,\middle\|\, p_{\theta^*}(x_{1:T})\right)
= 
O\left(
\frac{1}{T^2}
\right)
\]
where the big-$O$ hides terms that scale polynomially in $D$. If in contrast, we insist that the covariance employed in denoising, $\Sigma_{\theta}(\cdot,\cdot)$ be isotropic, i.e. $\Sigma_{\theta}(\cdot,\cdot) \propto I_d$ is some scaling of the identity matrix, it is known that the path KL is necessarily $\Omega(1/T)$, see \citet[Theorem~7]{chen2022sampling}. The result above thus quantifies the value of learning the posterior covariance. The conceptual idea made precise by the result is that errors in estimating the posterior mean contribute irreducible error to the path KL; errors in posterior covariance contribute $O(1/T)$ error; whereas errors in higher order tensor moments of the posterior contribute $O(1/T^2)$ error.  The left panel of Figure~\ref{fig:pathkl_toy_celeba_subset} provides a striking illustration of this theory in a simple model from \cite{ou2025ocm} where $q$ is a mixture of Gaussians.


\subsection{Drawing Gaussian Samples with the Optimal Posterior Covariance}
\label{sec:lanczos_setup}

How do we {\em practically} realize the value of matching the full covariance?  As discussed in the introduction, a robust literature now attempts to learn approximations to $\Sigma_{\theta^*}$; to manage the stability of learning and make storage feasible, these approaches typically learn a diagonal approximation to $\Sigma_{\theta^*}$. We take a different approach: observe we do not require $\Sigma_{\theta^*}$ directly and instead only need to sample from the $p_{\theta^*}(x_{t-1}|x_t)$. But such a sample can be obtained as 
\[
\mu_{\theta^*}(x_t,t) + \Sigma_{\theta^*}^{1/2}(x_t,t)Z
\] 
where $Z$ is a standard Gaussian so that our goal is to compute $\Sigma_{\theta^*}^{1/2}(x_t,t)Z$. Now it is known (via Tweedie's formula, and observed by \cite{efron2011tweedie,manor2023posterior,zhang2023moment}) that  
\[
\Sigma_{\theta^*}(x_t,t) 
=
\frac{\beta_t}{\sqrt{2-\beta_t}}
\nabla_x \mu_{\theta^*}(x_t,t)
\]
Consequently, $\Sigma_{\theta^*}(x_t,t)Z$ is easily computed with $O(d)$ storage and a single backward pass through the network defining $\mu_{\theta^*}(\cdot,\cdot)$. {\em Can we use this ability to compute  $\Sigma_{\theta^*}(x_t,t)^{1/2} Z$?} 

In Section~\ref{sec:lanczos} we present the {\em Lanczos Gaussian sampler}. This sampler uses a small number of so-called Lanczos iterates of the form 
$\Sigma_{\theta^*}(x_t,t)^kZ$ for $k=1,2,\dots, m$ to approximate  $\Sigma_{\theta^*}(x_t,t)^{1/2} Z$. We show in Theorem~\ref{thm:lanczos} that the approximation error of this procedure is $O(3^{-m})$. Our experiments in Section~\ref{sec:experiments} (see Table~\ref{tab:fid_all_datasets}) take $m=3,5$ and show non-trivial improvements over a SOTA approach (OCM DDPM) in several settings, notably {\em without requiring an auxiliary model}. The right panel of Figure~\ref{fig:pathkl_toy_celeba_subset} previews those experimental results for the CELEBA 64$\times$64 dataset.

%% file: sections/main_result.tex
\section{The Value of Covariance Learning}
\label{sec:approximation_error}

This section establishes our main result and provides an outline of the proof. Our proof is a departure from existing analysis techniques for the path KL \citep{li2024towards,benton2023nearly,chen2022sampling}
, and rests on an information theoretic toolkit (higher order so-called I-MMSE identities for the Gaussian channel). 

Recall that we defined (see~\eqref{eqn:diffusionLoss})
\[
\mathcal{L}(\theta)
\triangleq
\sum_{t=2}^T
\E_{q(x_t,x_0)}\!\left[
  \KL\!\left(q(x_{t-1}\mid x_t,x_0)\,\middle\|\,p_\theta(x_{t-1}\mid x_t)\right)
\right]
\]
and note that if $p_\theta$ were not constrained, the optimizer in~\eqref{eqn:diffusionLoss} is simply $p_\theta(x_{t-1}|x_t) = q(x_{t-1}|x_t)$. We define the optimal diffusion loss $\mathcal{L}^*$ according to 
\[
\mathcal{L}^*
\triangleq
\sum_{t=2}^T
\E_{q(x_t,x_0)}\!\left[
  \KL\!\left(q(x_{t-1}\mid x_t,x_0)\,\middle\|\,q(x_{t-1}\mid x_t)\right)
\right]
\]
Let $\mathcal{L}_t(\theta^*)$ and $\mathcal{L}_t^*$ denote the $t$-th summands of $\mathcal{L}(\theta^*)$ and $\mathcal{L}^*$. Our first observation is that the unconditional path KL can be identified with the difference between $\mathcal{L}(\theta)$ and $\mathcal{L}^*$. Specifically:
\begin{proposition}\label{prop:unconditional-path-kl}
    The unconditional path KL satisfies 
\[
\KL\!\left(q(x_{1:T}) \,\middle\|\, p_\theta(x_{1:T})\right)
=
\mathcal{L}(\theta) 
- \mathcal{L}^*
+ \KL\!\left(q(x_T) \,\middle\|\, \N(0,I) \right)
\]
\end{proposition}

We introduce some further (standard) notation. Let $\bar{\alpha}_t = \Pi_{s=1}^t 1-\beta_s$. Then, it follows by definition of the forward process~\eqref{eqn:forwardKernel} that $x_t = \sqrt{\bar{\alpha}_t}x_0 + \sqrt{1-\bar{\alpha}_t}Z$ where $Z \sim \N(0,I)$ is independent of $x_0$. Define the signal-to-noise ratio (SNR) of $x_t$ as ${\rm SNR}_t = {\bar{\alpha}_t}/{1-\bar{\alpha}_t}$ and denote by $\gamma_t = \text{SNR}_t - \text{SNR}_{t+1}$ the difference of consecutive ${\rm SNR}_t$. We can think of $\gamma_t = O(1/T)$ under common step-size schedules. We can now state our main result: 

\begin{theorem}\label{thm:mainTheorem}
    Assuming that $\norm{x_0} \leq D$ and $\gamma_t \leq \frac{C}{T}$ for some constant $C > 0$ then 
    \[
    \KL\!\left(q(x_{1:T}) \,\middle\|\, p_\theta(x_{1:T})\right) - \KL\!\left(q(x_T) \,\middle\|\, \N(0,I) \right) = \mathcal{L}(\theta^*) - \mathcal{L}^* \leq \frac{C^3D^6}{2T^2}
    \]
\end{theorem}


Our proof follows by computing the second-order Taylor expansion of $\mathcal{L}_t^*$ and $\mathcal{L}_t(\theta^*)$ with respect to $\gamma_t$. In particular, we show that the first and second order terms of these Taylor expansions are identical. Thus, the difference $\mathcal{L}_t(\theta^*) - \mathcal{L}_t^*$ is $O(\gamma_t^3)$. Taking an (approximately) uniform discretization of $[\text{SNR}_T, \text{SNR}_1]$ over $T-1$ steps such that $\gamma_t = O(1/T)$ yields the $O(1/T^2)$ result. 

\textbf{Taylor Expansion for $\mathcal{L}_t(\theta^*)$: }An analysis of the optimization problem defined by $\mathcal{L}_t(\theta^*)$ shows: 
\[
\mathcal{L}_t(\theta^*) = \frac{1}{2}\E_{q(x_t)}[\log\det(I + \gamma_t \Cov(x_0\mid x_t))]
\] 
Taking the Taylor expansion with respect to $\gamma_t$ and applying Taylor's theorem then yields: 
\begin{lemma} 
\label{le:gaussian_taylor}
    Assuming that $\norm{x_0} \leq D$,
    \[
        \mathcal{L}_t(\theta^*)
        \leq
        \frac{\gamma_t}{2}
        \E_{q(x_t)}
        \left[\tr\left(\Cov(x_0|x_t)\right)
        \right] - 
        \frac{\gamma_t^2}{4}
        \E_{q(x_t)}
        \left[\tr\left(\Cov(x_0|x_t)^2
        \right)\right] + 
        \frac{\gamma_t^3D^6}{6}
   \]
\end{lemma}

We turn next to the expansion for $\mathcal{L}^*_t$. 

\textbf{Taylor Expansion for $\mathcal{L}^*_t$: }We begin here with a simple but novel representation of $\mathcal{L}^*_t$ as a Gaussian channel mutual information. Doing so will then let us leverage a powerful information theoretic framework. To that end, consider the Gaussian channel $Y = \sqrt{\gamma_t} X_{0|t} + Z$ where $X_{0|t} \sim q(x_0|x_t)$. We have the following representation of $\mathcal{L}^*_t$:

\begin{proposition}
\label{prop:gaussian_channel_representation}
\[
\mathbb{E}_{q(x_t)}
\left[
I(X_{0|t};\sqrt{\gamma_t}X_{0|t}+Z)
\right]
=
\mathcal{L}^*_t
\]
\end{proposition}
\begin{proof}
Define
\[
U = 
\frac{\sqrt{\bar \alpha_t}(1-\bar{\alpha}_{t-1})}{1-\bar{\alpha}_t}
x_t
+
\frac{\sqrt{\bar{\alpha}_{t-1}}\beta_t}{1-\bar{\alpha}_t}
X_{0|t}
+ 
\tilde{\beta}_tZ
\]
where $Z \sim \N(0,I)$ and is independent of $X_{0|t}$ and $\tilde{\beta}_t = (1-\bar{\alpha}_{t-1})\beta_t / (1-\bar{\alpha}_t)$. Conditioned on $X_{0|t} = x_0$, $U$ is distributed according to $q(x_{t-1} \mid x_t, x_0)$. Marginalizing over $q(x_0 \mid x_t)$ it follows that $U\sim q(x_{t-1}\mid x_t)$.Then the mutual information KL divergence identity yields
\[
I(X_{0|t};U) = \E_{q(x_0\mid x_t)}[\KL(q(x_{t-1}\mid x_t,x_0) || q(x_{t-1}\mid x_t))]
\]
Now we observe that 
$
({U - a_t x_t})/{\tilde{\beta}^{1/2}_t}
= 
\sqrt{\gamma_t}X_{0|t} + Z. 
$
where $a_t = ({\sqrt{\bar \alpha_t}(1-\bar{\alpha}_{t-1})})/({1-\bar{\alpha}_t})$. 
Hence, $I(X_{0|t};U) = I(X_{0|t};\sqrt{\gamma_t}X_{0|t}+Z)$ by invariance under invertible transformations. Then


\[
I(X_{0|t};\sqrt{\gamma_t}X_{0|t} + Z) = \E_{q(x_0\mid x_t)}[\KL(q(x_{t-1}\mid x_t,x_0) || q(x_{t-1}\mid x_t))]
\]
and taking expectations over $x_t$ completes the proof. 
\end{proof}

Defining $g(\gamma) = I(X_{0|t};\sqrt{\gamma}X_{0|t}+Z)$ we observe that a Taylor approximation to $g(\gamma)$ around $\gamma = 0$ can be obtained via the I-MMSE identity and higher-order derivative identities with respect to the channel SNR~\cite{guo2005mutual,payaro2009hessian,nguyen2024derivatives}. This allows us to prove:  


\begin{lemma}
\label{le:exact_taylor}
    Assuming that $\norm{x_0} \leq D$,
    \[
        \mathcal{L}^*_t
        \geq 
        \frac{\gamma_t}{2}\E_{q(x_t)}[\tr\Cov(x_0|x_t)] - \frac{\gamma_t^2}{4}\E_{q(x_t)}[\tr(\Cov(x_0|x_t)^2)] - \frac{\gamma_t^3D^6}{3} 
    \]
\end{lemma}
 
Lemmas~\ref{le:gaussian_taylor} and~\ref{le:exact_taylor} together yield the proof of Theorem~\ref{thm:mainTheorem}.

\textbf{Path KL and conditional generation.} We close by noting that the path KL studied in this section is not only a natural measure of reverse-process accuracy, but also has direct implications for conditional generation procedures that modify the reverse trajectory. A canonical example is classifier guidance, where errors accumulated along the path can translate into error in the conditional endpoint distribution. In Appendix~\ref{sec:classifier_guidance}, we make this connection formal by showing that the unconditional path KL controls the expected conditional endpoint KL under exact classifier guidance, see Proposition~\ref{prop:classifier_guidance}. In Appendix~\ref{app:counter-example}, we show through a simple Gaussian example, see Theorem~\ref{thm:cg-counterexample-constant} that with a scalar reverse covariance, the conditional endpoint KL can retain a $\Omega(1/T)$, whereas the corresponding full-covariance reverse process achieves the faster $O(1/T^2)$ rate predicted by our Theorem~\ref{thm:mainTheorem}.

%% file: sections/lanczos.tex
\section{The Lanczos Gaussian Sampler}
\label{sec:lanczos}

A sample from the reverse kernel~\eqref{eqn:reverseKernel} can simply be generated by computing $\mu_\theta(x_t,t) + \Sigma_\theta(x_t,t)^{1/2} Z$. Our focus in this section is on computing  $\Sigma_\theta(x_t,t)^{1/2} Z$ efficiently. Recall from Section~\ref{sec:lanczos_setup} that 
$
\Sigma_{\theta}(x_t,t) 
=
({\beta_t}/{\sqrt{1-\beta_t}})
\nabla_x \mu_{\theta}(x_t,t)
$. Consequently, for any vector $v$, we have 
\[
\Sigma_\theta(x_t,t) v
=
\frac{\beta_t}{\sqrt{1-\beta_t}}
\nabla_x \mu_{\theta}(x_t,t)^\top v
\]
the latter of which can be computed efficiently with $O(d)$ storage. Succinctly then, our problem is: {\em how can we compute $\Sigma_\theta(x_t,t)^{1/2} Z$ given access to an oracle that can efficiently compute $\Sigma_\theta(x_t,t) x$ for any vector $x$?} To this end, we present a subroutine that given as input a vector $v$ and an oracle that can compute the product $A x$ for any vector $x$ (where $A$ is psd), returns an approximation to $A^{1/2}v$. 


\begin{algorithm}[h]
\caption{Krylov-Lanczos approximation to $A^{1/2}v$}
\label{alg:krylov-sqrt}
\begin{algorithmic}[1]
\Require vector $v$, Krylov dimension $m$, matvec oracle $x\mapsto Ax$
\Ensure approximation $y_m \approx A^{1/2}v$
\State Compute the Krylov set $\mathcal K_m(A,v) = \{v,Av,\dots,A^{m-1}v\}$. 
\State Construct an orthonormal basis
$
Q_m = \textsc{GramSchmidt}\big(\mathcal K_m(A,v)\big)
$
\State Form the projected matrix
$
T_m \gets Q_m^\top A Q_m .
$
\State Return
$
y_m \gets \|v\|_2 Q_m T_m^{1/2} e_1 .
$
\end{algorithmic}
\end{algorithm}
 
 Algorithm~\ref{alg:krylov-sqrt} takes as input an iteration budget $m$. Note that the matrices $T_m$ and $Q_m$ have dimension $m \times m$ and $d \times m$ respectively and so for small $m$ the algorithm is extremely practical and efficient. It is worth noting that since $A \succeq 0$ it turns out that $T_m$ is in fact a tri-diagonal matrix and can be computed via an extremely efficient (Lanczos) iteration; see Appendix~\ref{app:lanczos}. As such it turns out that an efficient implementation of Algorithm~\ref{alg:krylov-sqrt} requires precisely $m$ calls of the matvec oracle. We next analyze the approximation error afforded by the use of Algorithm~\ref{alg:krylov-sqrt}. Denote the output of Algorithm~\ref{alg:krylov-sqrt} by $y_m(A,v)$; we have:  

\begin{theorem}
\label{thm:lanczos}
Assume that $\|x_0\|\leq D$, $\tilde{\beta}_t \leq 1$, and $\gamma_t \leq C/T$ for some constant $C > 0$ and $T > D^2 C$. Then,
\[
\| y_m( \Sigma_{\theta^*}(x_t,t),Z) - \Sigma_{\theta^*}(x_t,t)^{1/2} Z \|
\leq 
4\sqrt{2\tilde{\beta}_t}(\sqrt{3} -1) \|Z\| 3^{-m} 
\leq
4\sqrt{2}(\sqrt{3} -1) \|Z\| 3^{-m} 
\]
\end{theorem}

The approximation error of Algorithm~\ref{alg:krylov-sqrt} decreases exponentially in $m$. As a result, in our experiments we get away with setting $m=3$ which in turn boils down to {\em just three additional backward passes through the $\mu_{\theta}(x_t,t)$ network}.

\subsection{Proof of Theorem~\ref{thm:lanczos}}

We invoke the exact-arithmetic Lanczos approximation guarantee of
\citet[Theorem~4.1]{musco2018stability},
specialized to the matrix function $f(A)=A^{1/2}$:


\begin{theorem}
[MMS18]
\label{thm:MMS18}
Let $A$ be positive semi-definite with condition number $\kappa$. Denote by $\mathcal P_m$ the set of polynomials of degree at most $m$. Then for any $v$:
\[
\|y_m(A,v) - A^{1/2}v\| 
\leq
2 \sqrt{\lambda_{\rm min}(A) }\|v\| 
\min_{p \in \mathcal P_m}
\left(
\max_{x \in [1, \kappa]}
|
\sqrt{x} - p(x)
|
\right)
\]
\end{theorem}

For brevity we define $\kappa := \kappa(\Sigma_{\theta^*}(x_t,t))$ and $\lambda_{\rm min} := \lambda_{\rm min}(\Sigma_{\theta^*}(x_t,t))$. We start by bounding 
$
\min_{p \in \mathcal P_m}
\left(
\max_{x \in [1, \kappa]}
|
\sqrt{x} - p(x)
|
\right)
$. To do so, we rely on the following classical Bernstein ellipse estimate for Chebyshev approximation; see, e.g.,
\citet[Theorem~8.2]{trefethen2019approximation}.



\begin{theorem}
[Bernstein ellipse estimate]
\label{thm:bernstein}
Let $\rho>1$, and let $E_\rho$ denote the open Bernstein ellipse
with foci $\{-1,1\}$. Suppose $f$ is analytic on $[-1,1]$ and admits
an analytic continuation to $E_\rho$. Assume moreover that
$
    |f(z)| \le M
    \text{\ for all \ } z\in E_\rho .
$
Let $f_m$ be the degree-$m$ Chebyshev projection of $f$ on $[-1,1]$. Then
\[
    \|f-f_m\|_{\infty,[-1,1]}
    \le
    \frac{2M\rho^{-m}}{\rho-1}.
\]
\end{theorem}  

Consider the real function
$
    f(x) := \sqrt{1+{\kappa}/{2}+{\kappa x}/{2}}
$. 
This function is real analytic on \([-1,1]\). We extend
it to the complex plane by taking the principal branch of the square root:
$
    f(z) := \sqrt{1+{\kappa}/{2}+ {\kappa z}/{2}}.
$
This gives an analytic continuation of \(f\) to the slit domain
$
    \mathbb C \setminus \left(-\infty,-1-{2}/{\kappa}\right]
$.

Now the Bernstein ellipse $E_\rho$ with foci $\{-1,1\}$ has a semi-major axis of length $(\rho + \rho^{-1})/2$ so that its leftmost point is $-(\rho + \rho^{-1})/2$. Consequently, for $f$ to be analytically continuable over $E_\rho$ it suffices that $-(\rho + \rho^{-1})/2 > -1-{2}/{\kappa}$. This is equivalent to requiring $\rho < (\sqrt{1+\kappa} + 1)/(\sqrt{1+\kappa} - 1)$. Moreover, observe that $\max_{z \in E_\rho} |f(z)| = \sqrt{1 + \kappa/2 + \kappa (\rho + \rho^{-1})/4}$ where the maximum is attained at $z =  (\rho + \rho^{-1})/2$. Applying Theorem~\ref{thm:bernstein} and taking a limit of the bound it yields as $\rho \uparrow (\sqrt{1+\kappa} + 1)/(\sqrt{1+\kappa} - 1)$ the: 

\begin{corollary}
\label{cor:bernstein_sqrt}
\[
\min_{p \in \mathcal P_m}
\left(
\max_{x \in [1, \kappa]}
|
\sqrt{x} - p(x)
|
\right)
\leq
\sqrt{\kappa+2}\,
\frac{\left(\sqrt{1+\kappa}-1\right)^{m+1}}
{\left(\sqrt{1+\kappa}+1\right)^m}.
\] 
\end{corollary}

To conclude the proof of Theorem~\ref{thm:lanczos} it suffices to bound $\kappa$ and $\lambda_{\rm min}$. To this end, we observe that $\Sigma_{\theta^*}(x_t,t) = \tilde{\beta}_t(I + \gamma_t\Cov(x_0|x_t))$. Hence, $\kappa \leq 1 + \gamma_t \lambda_{\rm max}({\rm Cov}(x_0|x_t)) \leq 1 + \gamma_t \| x_0\|^2$. Under the assumption that $\gamma_t \leq C/T$ and $T > CD^2$ then, $\kappa \leq 2$. Similarly, $\lambda_{\rm min} \leq 2\tilde{\beta}_t$. These facts, together with the result of Theorem~\ref{thm:MMS18} and Corollary~\ref{cor:bernstein_sqrt} yields the result. 


%% file: sections/experiments.tex
\section{Experiments}\label{sec:experiments}

We complement the theoretical discussion so far with empirical evaluations of diffusion sampling under various posterior covariance formulations. In Section~\ref{subsec:toy-experiments} we present a 2D toy example that validates our $O(1/T^2)$ path KL rate under optimal covariance matching, and simultaneously, demonstrates an $\Omega(1/T)$ rate for both the isotropic and optimal diagonal covariance parameterizations. Section~\ref{subsec:image-ddpm} is focused on the Lanczos Gaussian sampler (LGS). We study image generation tasks on various datasets and demonstrate the relative merits of LGS vis-a-vis a SOTA alternative (OCM) and isotropic covariances.


\subsection{Toy Experiment: Exact Path KL Evaluations}
\label{subsec:toy-experiments}

In order to illustrate our theoretical analysis of the path KL under DDPM sampling we consider a simple 2D toy example orgiinally presented in \citet{ou2025ocm}. We consider a data disribution $q(x_0)$ that is a mixture of 40 Gaussians with means sampled uniformly on $[-40,40]\times [-40,40]\subseteq \R^2$ and with standard deviation $\sigma = \sqrt{40}$. This choice of data distribution, $q$ allows us to compute the reverse kernel's conditional mean, $\mu^*(x_t,t)$, and covariance, $\Sigma^*(x_t,t)$, in closed form which so that an essentially closed form expression of each summand in the path KL is available. 

We vary the number of denoising steps $T$ and compute the path KL between the exact reverse process and the learned reverse process under three different covariance formulations: (1) isotropic covariance with variance fixed to $\tilde{\beta}_t$, (2) optimal diagonal covariance and (3) the optimal full covariance. In the left panel of Figure~\ref{fig:pathkl_toy_celeba_subset} we plot the results on a log-log scale to visualize the rate of decay of the path KL as a function of $T$. In addition we plot reference curves $T^{-1}$ and $T^{-2}$.
We see from the left panel of Figure~\ref{fig:pathkl_toy_celeba_subset} that sampling with the optimal (full) reverse kernel covariance achieves the $O(1/T^2)$ rate predicted by our analysis, {\em while sampling with an isotropic or (optimal) diagonal covariance yields only a $\Omega(1/T)$ rate}. These results corroborate our Theorem~\ref{thm:mainTheorem} and the $\Omega(d/T)$ lower bound due to \citet{chen2022sampling}. Importantly, there is little analysis available on the use of diagonal approximations, but these results suggest that while they improve on the use of an isotropic covariance in an absolute sense, they do not offer an order improvement in $T$ (i.e the path KL still retains $\Omega(1/T)$ error). 

\subsection{Image modeling with diffusion models}
\label{subsec:image-ddpm}

We follow standard DDPM evaluation protocols \citep{ho2020denoising,bao2022analytic, ou2025ocm} and report sample quality using Fr\'echet inception distance (FID). Note that our choice of FID is driven by the fact that since LGS yields {\em samples} from the reverse posterior (as opposed to a posterior itself), any likelihood computation would be sample-based and thus infeasible. 

We consider the following datasets: \textsc{CIFAR-10} \citep{ho2020denoising,nichol2021improved}, \textsc{CelebA 64x64} \citep{liu2015deep}, and \textsc{ImageNet} at $64\times 64$.
For each dataset we employ a linear $\beta_t$ step size schedule and compare the following methods: isotropic reverse kernels with variance fixed to either $\beta_t$ or $\tilde{\beta}_t$ (``DDPM, $\beta$'' / ``DDPM, $\tilde{\beta}$''); OCM-DDPM \citep{ou2025ocm}, that learns the optimal diagonal covariance; and our approach, LGS-DDPM, (Section~\ref{sec:lanczos}) with Krylov ranks $m\in\{3,5\}$. 

For all models we use the pre-trained score networks from previous works (\citep{bao2022analytic, nichol2021improved,song2020denoising, ou2025ocm}). Importantly, we note that OCM-DDPM requires training an auxiliary model that predicts the posterior covariance diagonal whereas LGS-DDPM (and of course, DDPM, $\beta$ / DDPM, $\tilde{\beta}$) only require the score model. Finally, we consider experiments for $T=25,50$ and $100$ time steps having observed from \cite{ou2025ocm} that FID plateaus after $T\sim100$. 

Finally, we note that standard errors computed by bootstrapping from the 50k generated images in each experiment are negligible. Note however, that training new score models on bootstraps of the training data can introduce variability; here we train with the single checkpoints from \citep{bao2022analytic, nichol2021improved, song2020denoising, ou2025ocm}. 

\subsubsection{Inference-Time Speedup}
OCM-DDPM must make two forward passes, one through the score network and the second through the diagonal covariance prediction network in each time step. On the other hand, LGS-DDPM must make a single forward pass through the score network, and either $3$ or $5$ backwards passes (recall $m\in\{3,5\}$ here). We implement some simple ideas that yield enough speedup so that LGS-DDPM takes roughly the same amount of wall-clock time as OCM-DDPM to generate a single sample.

\textbf{Batching: } We reduce the per-step cost of LGS-DDPM by amortizing Lanczos computations across nearby time steps. At time step $t$, instead of drawing only one covariance-matched noise term, we compute $\Sigma^{1/2}(x_t,t) Z_1, \dots, \Sigma^{1/2}(x_t,t) Z_{l}$ in parallel for independent Gaussian draws $Z_1, \dots, Z_{l}$ and reuse these samples over the steps $t, t-1, \dots, t-l+1$. This introduces a mildly stale covariance for the later steps in the block, but the effect is negligible when the step size is small. In return, the additional backward passes required by the Lanczos sampler are incurred only once every $l$ steps.


In addition to the batching above, we find that the empirical value of covariance matching is largely realized in later time steps. As such LGSb-DDPM in Table~\ref{tab:fid_all_datasets} impelements batching with $l=2$ and applies the algorithms only to the last $25\%$ of time steps. As such, we find that LGSb-DDPM takes roughly equal wall-clock time as OCM DDPM per sample generation.

\textbf{FID at matched wall-clock budget: } To fairly compare sample quality at a fixed wall-clock budget rather than a fixed step count, we measure the average per-sample wall-clock time of LGSb-DDPM and LGS-DDPM (both with $m=3,l=2$) at $T=50$ and, for each dataset, estimate the $T^*$ at which OCM-DDPM incurs the same wall-clock cost. OCM-DDPM is thus allowed a larger number of time-steps, $T^*=67$ and $T^*=113$ respectively. Table~\ref{tab:fid_matched_effort} reports FID at this matched budget. 

\subsubsection{Conclusions}
We draw the following conclusions from our results in Tables~\ref{tab:fid_all_datasets}:

\textbf{LGS-DDPM requires small $m$: }We see from Table~\ref{tab:fid_all_datasets} that the gain in FID in moving from $m=3$ to $m=5$ is marginal. This comports with our theory in Section~\ref{sec:lanczos} where we showed that the error in approximating $\Sigma^{1/2} Z$ falls exponentially in $m$. This is crucial to the practicality of LGS-DDPM: its use requires a small number of additional backward passes through the score network, and as discussed these can be further amortized via batching. 

\textbf{LGS-DDPM vs. OCM-DDPM: }In all of our experimental setups a variant of LGS-DDPM yields the lowest FID score. Again, we emphasize that OCM DDPM requires training an additional model to predict the covariance diagonal. 

\textbf{Batching can improve FID: } LGS-DDPM with $l=2$ achieves the best FID in 5 out of 9 settings, including LGSb-DDPM. This is somewhat unexpected since batching introduces staler covariance estimates relative to $l=1$. We conjecture that the effect may arise from an implicit averaging across LGS samples, although a theoretical explanation remains open.

\textbf{FID at matched effort: } Recall that in Table~\ref{tab:fid_matched_effort} we allow OCM-DDPM additional time steps to match wall click time with LGS-DDPM and LGSb-DDPM precisely. We see there that LGSb-DDPM and LGS-DDPM both attain lower FID than OCM-DDPM on \textsc{CIFAR-10} and \textsc{CelebA}, while remaining competitive on \textsc{ImageNet} (note that \cite{ou2025ocm} (Table 9) report a standard error of $0.97$ at $T=50$ for \textsc{ImageNet} across variation in their learned covariance network but only provide a single checkpoint).

\begin{table}[t]
\centering
\caption{FID score $\downarrow$ across datasets using different sampling steps.}
\label{tab:fid_all_datasets}
{\small
\setlength{\tabcolsep}{4.8pt}%
\setlength{\aboverulesep}{0.4ex}%
\setlength{\belowrulesep}{0.4ex}%
\renewcommand{\arraystretch}{1.0}%
\begin{tabular*}{\linewidth}{@{\extracolsep{\fill}}lccccccccc@{}}
\toprule
& \multicolumn{3}{c}{CIFAR-10}
& \multicolumn{3}{c}{\textsc{CelebA}~64$\times$64}
& \multicolumn{3}{c}{\textsc{ImageNet}~64$\times$64} \\
\cmidrule(lr){2-4}
\cmidrule(lr){5-7}
\cmidrule(lr){8-10}
\# timesteps $T$
& 25 & 50 & 100
& 25 & 50 & 100
& 25 & 50 & 100 \\
\midrule
DDPM, $\beta$
& 126.22 & 67.55 & 31.50
& 116.16 & 55.39 & 29.23
& 170.13 & 83.86 & 44.86 \\
DDPM, $\tilde{\beta}$
& 21.84 & 15.06 & 10.84
& 24.73 & 18.76 & 14.26
& 29.22 & 21.88 & 19.17 \\
OCM-DDPM
& 9.39 & 5.83 & 4.31
& 12.66 & 9.16 & 7.09
& 27.82 & 20.85 & \underline{18.09} \\
\midrule
LGS-DDPM\phantom{b} $(m=3,l=1)$
& \textbf{5.06} & 5.08 & 4.35
& 9.58 & 7.01 & 5.21
& \underline{22.13} & \textbf{20.20} & 18.69 \\
LGS-DDPM\phantom{b} $(m=5,l=1)$
& \underline{5.09} & \underline{4.94} & \underline{4.25}
& \underline{8.84} & \textbf{6.29} & \underline{4.79}
& \textbf{21.93} & \underline{20.24} & 18.72 \\
LGS-DDPM\phantom{b} $(m=3,l=2)$
& 6.51 & \textbf{3.85} & \textbf{3.99}
& \textbf{8.68} & \underline{6.38} & \textbf{4.09}
& 29.61 & 22.28 & 18.57 \\
LGSb-DDPM $(m=3,l=2)$
& 13.40 & 5.40 & 4.88
& 16.89 & 6.69 & 5.21
& 39.72 & 21.19 & \textbf{18.08} \\
\bottomrule
\end{tabular*}
}
\end{table}

\begin{table}[t]
\centering
\caption{FID score $\downarrow$ under matched wall-clock time per sample generation settings across datasets.}
\label{tab:fid_matched_effort}
{\small
\setlength{\tabcolsep}{5.5pt}%
\renewcommand{\arraystretch}{1.0}%
\begin{tabular*}{\linewidth}{@{\extracolsep{\fill}}lccc@{}}
\toprule
Model & CIFAR-10 & \textsc{CelebA}~$64\times 64$ & \textsc{ImageNet}~$64\times 64$ \\
\midrule
OCM-DDPM, $T=67$ & 6.10 & 8.10 & \textbf{19.23} \\
LGSb-DDPM $(m=3,l=2)$, $T=50$ & \textbf{5.40} & \textbf{6.69} & 21.19 \\
\midrule
OCM-DDPM, $T=118$ & 5.45 & 6.67 & \textbf{17.60} \\
LGS-DDPM $(m=3,l=2)$, $T=50$ & \textbf{3.85} & \textbf{6.38} & 22.28 \\
\bottomrule
\end{tabular*}
}
\end{table}


%% file: sections/conclusion.tex
\section{Conclusion}

In this paper we study the value of covariance matching in Gaussian DDPMs from both theoretical and algorithmic perspectives. On the theoretical side, we show that matching the full posterior covariance breaks the standard $O(1/T)$ path-KL barrier: optimal covariance matching achieves an $O(1/T^2)$ path KL rate, see Theorem~\ref{thm:mainTheorem}. We further show that this improvement is consequential for classifier guidance, where using a scalar covariance can yield an $\Omega(1/T)$ conditional endpoint-KL error, while access to the optimal covariance restores the $O(1/T^2)$ rate.

On the algorithmic side, we introduce the Lanczos Gaussian sampler, a training-free method for sampling from the full posterior covariance using only covariance-vector products obtained from Jacobian-vector products of a pretrained score network. The method applies to any pretrained Gaussian DDPM without auxiliary covariance training or dense covariance storage, and we show that its approximation error decays exponentially in the number of Lanczos steps. Across standard image-generation benchmarks, LGS consistently improves sample quality over isotropic and diagonal-covariance baselines, demonstrating that full covariance matching is not only theoretically beneficial but also practically accessible for fast diffusion sampling.

%% file: sections/appendix.tex
\section{The Value of Covariance Matching in Classifier Guided DDPMs}\label{sec:classifier_guidance}

We end this Section by asking why one might care about the {\em path} KL as opposed to simply say the KL at the endpoint, $\KL\!\left(q(x_{1}) \,\middle\|\, p_\theta(x_{1})\right)$. Here we note that the path KL is relevant to algorithms that rely on faithful reconstruction of the entire path. One such example is classifier guided diffusion modeling~\citep{dhariwal2021diffusion}. Specifically, assume that each $x_0$ is associated with a label $y$ so that we are endowed with a distribution $q(x_0,y)$; our goal is to sample from $q(x_0|y)$. Under classifier guidance, the conditional analog to the reverse process~\eqref{eqn:reverseKernel} is 
\begin{equation*}
    p_\theta(x_{t-1}\mid x_t, y) = p_\theta(x_{t-1}\mid x_t) \frac{p_\theta(y\mid x_{t-1})}{p_\theta(y\mid x_t)}
\end{equation*} 
where the classifier $p_\theta(y\mid x_t)$ approximates the exact classifier $q(y\mid x_t)$. We have that under (exact) classifier guidance, the unconditional path KL upper bounds the expected KL of conditionally generated end point samples:  
\begin{proposition}
\label{prop:classifier_guidance}    
Under exact classifier guidance (i.e. $p_\theta(y\mid x_t) = q(y\mid x_t)$) and matching label priors
$p_\theta(y)=q(y)$, 
\[
\E_{q(y)}[
\KL\!\left(q(x_{1}\mid y) \,\middle\|\, 
p_\theta(x_{1} \mid y)\right)] 
\leq 
\KL\!\left(q(x_{1:T}) \,\middle\|\, p_\theta(x_{1:T})\right)
\]
\end{proposition} 

In the next Section we show that this conditional end point KL behaves like $\Omega(1/T)$ under a popular practical algorithm for classifier guidance while the same algorithm equipped with an optimally learned covariance enjoys $O(1/T^2)$ conditional end point KL.

\textbf{Proof of Proposition \ref{prop:classifier_guidance}: }

\begin{proof}
By data processing,
\[
\E_{q(y)}\left[\KL(q(x_1\mid y)\|p_\theta(x_1\mid y))\right]
\le
\E_{q(y)}\left[\KL(q(x_{1:T}\mid y)\|p_\theta(x_{1:T}\mid y))\right].
\]
Thus it suffices to analyze the right-hand side. Because the forward process is label-independent,
\[
q(x_{t-1}\mid x_t,y)
=
q(x_{t-1}\mid x_t)
\frac{q(y\mid x_{t-1})}{q(y\mid x_t)}.
\]
Exact Bayesian classifier guidance gives
\[
p_\theta(x_{t-1}\mid x_t,y)
=
p_\theta(x_{t-1}\mid x_t)
\frac{q(y\mid x_{t-1})}{q(y\mid x_t)}.
\]
Hence
\[
\frac{q(x_{t-1}\mid x_t,y)}
     {p_\theta(x_{t-1}\mid x_t,y)}
=
\frac{q(x_{t-1}\mid x_t)}
     {p_\theta(x_{t-1}\mid x_t)}.
\]
The same Bayes-rule cancellation at $T$, using
$p_\theta(y\mid x_T)=q(y\mid x_T)$ and $p_\theta(y)=q(y)$, gives
\[
\frac{q(x_T\mid y)}{p_\theta(x_T\mid y)}
=
\frac{q(x_T)}{p_\theta(x_T)}.
\]
Applying the above identities to the factorization of the path distributions yields
\[
\frac{q(x_{1:T}\mid y)}{p_\theta(x_{1:T}\mid y)}
=
\frac{q(x_{1:T})}{p_\theta(x_{1:T})}.
\]
Therefore 
\[
\E_{q(y)}\left[\KL(q(x_{1:T}\mid y)\|p_\theta(x_{1:T}\mid y))\right]
= 
\KL(q(x_{1:T})\|p_\theta(x_{1:T}))
\]
thus concluding the proof.

\end{proof}

%% file: sections/appendix/appendix_counter_example.tex
\section{Counter Example}\label{app:counter-example}

This section gives a Gaussian example in which scalar-covariance classifier guidance has a conditional endpoint KL error in $\Omega(1)$, whereas the same guidance rule with the fixed unconditional
optimal covariance has conditional endpoint KL error in $O(1/T^2)$. 

Let $R\in\mathbb R^{2\times 2}$ be a non-trivial orthogonal rotation. 
Let
\[
Y_0=(Y_{0,1},Y_{0,2}),\qquad
Y_{0,1}\sim N(0,1),\qquad
Y_{0,2}\sim N(0,2),
\]
with independent coordinates, and set $X_0=RY_0$. Let the target conditional distribution be
\[
\mathsf C=Y_{0,1}+\varepsilon,\qquad \varepsilon\sim N(0,1),
\]
where $\varepsilon$ is independent of $Y_0$. We fix the total number of denoising steps $T$ and condition on $\mathsf C = c_T :=\sqrt{T}$. Then
\[
Y_{0,1}\mid \mathsf C=c_T \sim N\left(\frac{c_T}{2},\frac12\right).
\]

Define, for $t=1,\ldots,T$,
\[
\theta_t=T^{-4}+\frac{T-t}{T-1}(1-T^{-4}),
\qquad
\bar\alpha_t=\frac{\theta_t}{1+\theta_t},
\]
and set
\[
\beta_1=1-\bar\alpha_1,
\qquad
\beta_t=1-\frac{\bar\alpha_t}{\bar\alpha_{t-1}},\quad t=2,\ldots,T.
\]
Then the SNR increments satisfy the assumption of Theorem~\ref{thm:mainTheorem}.
\[
\mathrm{SNR}_t=\frac{\bar\alpha_t}{1-\bar\alpha_t}=\theta_t,
\qquad
\mathrm{SNR}_t-\mathrm{SNR}_{t+1}=\frac{1-T^{-4}}{T-1}\le \frac{2}{T}.
\]
The scalar-covariance DDPM uses
\[
\Sigma_t^{\rm sc}=\widetilde\beta_t I_2,
\]
whereas the fixed unconditional optimal covariance is
\[
\Sigma_t^\star
=R
\begin{pmatrix}
\beta_t & 0\\
0 & \displaystyle \beta_t\frac{1+\bar\alpha_{t-1}}{1+\bar\alpha_t}
\end{pmatrix}
R^\top .
\]
For both choices of covariance , the classifier-guided denoiser from \cite{dhariwal2021diffusion} uses the Gaussian conditional distributions
\[
X_{t-1}^{\widetilde\beta}\mid X_t^{\widetilde\beta}=x,\mathsf C=c_T
\sim
N\left(
\mu_t^0(x)+\frac{\widetilde\beta_t}{\sqrt{\alpha_t}}
\nabla_x\log q(\mathsf C=c_T\mid X_t=x),\ \widetilde\beta_t I_2
\right),
\quad t=T,\ldots,2.
\]
and
\[
X_{t-1}^\Sigma\mid X_t^\Sigma=x,\mathsf C=c_T
\sim
N\left(
\mu_t^0(x)+\frac{1}{\sqrt{\alpha_t}}\Sigma_t
\nabla_x\log q(\mathsf C=c_T\mid X_t=x),\ \Sigma_t
\right),
\quad t=T,\ldots,2,
\]
where $\mu_t^0(x)$ is the exact unconditional Gaussian reverse mean and both chains are initialized from $N(0,I_2)$ at time $T$.

\begin{theorem}[]
\label{thm:cg-counterexample-constant}
There are constants $c_0,c_1>0$, independent of $T$, such that for all sufficiently large $T$,
\[
\mathrm{KL}\Bigl(q(x_1\mid \mathsf C=c_T)\,\Big\|\,\mathcal L(X_1^{\widetilde\beta}\mid \mathsf C=c_T)\Bigr)
\ge \frac{c}{T},
\qquad
\mathrm{KL}\Bigl(q(x_1\mid \mathsf C=c_T)\,\Big\|\,\mathcal L(X_1^\Sigma\mid \mathsf C=c_T)\Bigr)
\le \frac{C}{T^2}.
\]
\end{theorem}

\begin{proof}
    KL is invariant under the rotation $x=Ry$, so we work in the $y$-coordinates. The classifier depends
    only on $Y_{0,1}$, and the two coordinates remain independent in this basis. Write
    $Y_t^{\widetilde\beta}=R^\top X_t^{\widetilde\beta}$ and $Y_t^\Sigma=R^\top X_t^\Sigma$.
    
    For the first coordinate of the true forward process,
    \[
    Y_{t,1}\mid \mathsf C=c_T\sim N(M_t,\ell_t),
    \qquad
    M_t=\frac{c_T}{2}\sqrt{\bar\alpha_t},
    \qquad
    \ell_t=1-\frac{\bar\alpha_t}{2}.
    \]
    The exact conditional reverse distribution is
    \[
    Y_{t-1,1}\mid Y_{t,1}=y,\mathsf C=c_T
    \sim
    N\left(M_{t-1}+a_t(y-M_t),\ \beta_t\frac{\ell_{t-1}}{\ell_t}\right),
    \qquad
     a_t=\sqrt{\alpha_t}\frac{\ell_{t-1}}{\ell_t}.
    \]
    Moreover,
    \[
    \mathsf C\mid Y_{t,1}=y\sim N\bigl(\sqrt{\bar\alpha_t}\,y,\ 1+u_t\bigr),
    \]
    so the one-dimensional classifier score is
    \[
    g_t(y)=\frac{\partial}{\partial y}\log q(\mathsf C=c_T\mid Y_{t,1}=y)
    =\frac{\sqrt{\bar\alpha_t}}{1+u_t}\bigl(c_T-\sqrt{\bar\alpha_t}\,y\bigr).
    \]
    Thus the first-coordinate guided updates are
    \[
    Y_{t-1,1}^{\widetilde\beta}\mid Y_{t,1}^{\widetilde\beta}=y,\mathsf C=c_T
    \sim
    N\left(\sqrt{\alpha_t}\,y+\frac{\widetilde\beta_t}{\sqrt{\alpha_t}}g_t(y),\ \widetilde\beta_t\right)
    \]
    and
    \[
    Y_{t-1,1}^{\Sigma}\mid Y_{t,1}^{\Sigma}=y,\mathsf C=c_T
    \sim
    N\left(\sqrt{\alpha_t}\,y+\frac{\beta_t}{\sqrt{\alpha_t}}g_t(y),\ \beta_t\right).
    \]
    
    We will use two elementary identities from the schedule. Since
    $\theta_{t-1}-\theta_t=(1-T^{-4})/(T-1)$ and $u_t=(1+\theta_t)^{-1}$,
    \[
    \beta_t=\frac{\theta_{t-1}-\theta_t}{\theta_{t-1}(1+\theta_t)},
    \qquad
    \beta_t-\widetilde\beta_t=\beta_t(\theta_{t-1}-\theta_t)u_{t-1}.
    \tag{1}
    \]
    Consequently, on the block
    \[
    \mathcal I_T=\{t\in\{2,\ldots,T\}:1/4\le \theta_t\le 3/4\},
    \]
    we have $|\mathcal I_T|\ge cT$ and
    \[
    \beta_t\asymp T^{-1},
    \qquad
    \beta_t-\widetilde\beta_t\asymp T^{-2}.
    \tag{2}
    \]
    
    We first prove the full-covariance upper bound by directly examining the error in the mean and variance at the endpoint. Since $1+u_t=2\ell_t$, $c_T=2(c_T/2)$, and
    $\ell_t=\alpha_t\ell_{t-1}+\beta_t$, direct simplification gives
    \[
    \sqrt{\alpha_t}\,y+\frac{\beta_t}{\sqrt{\alpha_t}}g_t(y)
    =M_{t-1}+a_t(y-M_t).
    \tag{3}
    \]
    Thus $Y_t^\Sigma$ has the exact conditional reverse mean in the first coordinate. Since
    $\mathbb E[Y_{T,1}^\Sigma\mid \mathsf C=c_T]=0$,
    \[
    \mathbb E[Y_{1,1}^\Sigma\mid \mathsf C=c_T]-M_1
    =-M_T\prod_{j=2}^T a_j = -M_T\sqrt{\frac{\bar\alpha_T}{\bar\alpha_1}}\frac{\ell_1}{\ell_T}.
    \]
    Because $\bar\alpha_T\le T^{-4}$ and $c_T=\sqrt T$,
    \[
    \left|\mathbb E[Y_{1,1}^\Sigma\mid \mathsf C=c_T]-M_1\right|
    \le C\sqrt T\,\bar\alpha_T
    \le CT^{-7/2}.
    \tag{4}
    \]
    We now look at the error in the variance. Let $r_t^\Sigma=\mathrm{Var}(Y_{t,1}^\Sigma\mid \mathsf C=c_T)$ and
    $d_t^\Sigma=r_t^\Sigma-\ell_t$. From (3),
    \[
    d_{t-1}^\Sigma
    =a_t^2d_t^\Sigma+\beta_t\left(1-\frac{\ell_{t-1}}{\ell_t}\right)
    =a_t^2d_t^\Sigma+\frac{\beta_t^2\bar\alpha_{t-1}}{2\ell_t},
    \qquad d_T^\Sigma=\frac{\bar\alpha_T}{2}.
    \tag{5}
    \]
    Also
    \[
    \prod_{j=2}^{t-1}a_j^2
    =\frac{\bar\alpha_{t-1}}{\bar\alpha_1}\left(\frac{\ell_1}{\ell_{t-1}}\right)^2
    \le C\bar\alpha_{t-1}.
    \]
    Iterating (5) and using (1),
    \[
    0\le d_1^\Sigma
    \le C\bar\alpha_T^2+C\sum_{t=2}^T\beta_t^2\bar\alpha_{t-1}^2
    \le C\bar\alpha_T^2+C\sum_{t=2}^T
    \frac{(\theta_{t-1}-\theta_t)^2}{(1+\theta_t)^2(1+\theta_{t-1})^2}
    \le \frac{C}{T}.
    \tag{6}
    \]
    The Gaussian KL formula, together with (4), (6), and the fact that $\ell_1$ is bounded above and
    below by positive constants, gives
    \[
    \mathrm{KL}\Bigl(q(y_{1,1}\mid \mathsf C=c_T)\,\Big\|\,
    \mathcal L(Y_{1,1}^\Sigma\mid \mathsf C=c_T)\Bigr)
    \le \frac{C}{T^2}.
    \tag{7}
    \]
    The second coordinate is independent of $\mathsf C$. Under $\Sigma_t$, its reverse mean and variance
    are the exact unconditional Gaussian reverse mean and variance. The only error is the terminal
    initialization which gives a second-coordinate endpoint KL of $O(\bar\alpha_T^2) = O(T^{-4})$. Combining this with (7) and rotating back to the
    $x$-coordinates proves
    \[
    \mathrm{KL}\Bigl(q(x_1\mid \mathsf C=c_T)\,\Big\|\,
    \mathcal L(X_1^\Sigma\mid \mathsf C=c_T)\Bigr)
    \le \frac{C}{T^2}.
    \]
    
    We now prove the scalar-covariance lower bound. Define
    \[
    e_t=\mathbb E[Y_{t,1}^{\widetilde\beta}\mid \mathsf C=c_T]-M_t.
    \]
    Using (3), the scalar guided mean satisfies
    \[
    e_{t-1}=A_te_t-F_t,
    \tag{8}
    \]
    where
    \[
    A_t=\sqrt{\alpha_t}-\frac{\widetilde\beta_t\bar\alpha_t}{\sqrt{\alpha_t}(1+u_t)},
    \qquad
    F_t=\frac{\beta_t-\widetilde\beta_t}{\sqrt{\alpha_t}}g_t(M_t).
    \tag{9}
    \]
    Since $g_t(M_t)=\sqrt{\bar\alpha_t}\,c_T/2$, (2) implies that
    \[
    F_t\ge \frac{c}{T^{3/2}},\qquad t\in\mathcal I_T.
    \tag{10}
    \]
    Furthermore,
    \[
    A_t=\sqrt{\alpha_t}\left(1-\frac{(\theta_{t-1}-\theta_t)u_{t-1}^2}{1+u_t}\right),
    \tag{11}
    \]
    so $A_t\ge0$ and, for every $t\in\mathcal I_T$,
    \[
    \prod_{j=2}^{t-1}A_j = \sqrt{\bar\alpha_{t-1}/\bar\alpha_1} \prod_{j=2}^{t-1}\left(1-\frac{(\theta_{j-1}-\theta_j)u_{j-1}^2}{1+u_j}\right)\ge c.
    \tag{12}
    \]
    
    Since $e_T=-M_T\le0$, unrolling (8) and using (10)--(12) yields
    \[
    -e_1
    \ge \sum_{t\in\mathcal I_T}F_t\prod_{j=2}^{t-1}A_j
    \ge \frac{c}{\sqrt T}.
    \tag{13}
    \]
    Let $r_t^{\widetilde\beta}=\mathrm{Var}(Y_{t,1}^{\widetilde\beta}\mid \mathsf C=c_T)$. Since
    $0\le A_t\le\sqrt{\alpha_t}$ and $\widetilde\beta_t\le\beta_t$,
    \[
    r_{t-1}^{\widetilde\beta}=A_t^2r_t^{\widetilde\beta}+\widetilde\beta_t
    \le \alpha_t r_t^{\widetilde\beta}+\beta_t.
    \]
    Starting from $r_T^{\widetilde\beta}=1$, induction gives $r_t^{\widetilde\beta}\le1$ for all $t$.
    The first-coordinate Gaussian KL and (13) therefore give
    \[
    \mathrm{KL}\Bigl(q(y_{1,1}\mid \mathsf C=c_T)\,\Big\|\,
    \mathcal L(Y_{1,1}^{\widetilde\beta}\mid \mathsf C=c_T)\Bigr)
    \ge
    \frac{e_1^2}{2r_1^{\widetilde\beta}}
    \ge \frac{c}{T}.
    \]
   Hence by rotating back to the $x$-coordinates and marginalizing over the second coordinate, we obtain
   \[
   \mathrm{KL}\Bigl(q(x_1\mid \mathsf C=c_T)\,\Big\|\,
   \mathcal L(X_1^{\widetilde\beta}\mid \mathsf C=c_T)\Bigr)
   \ge \frac{c}{T}.
   \],
   which completes the proof.
    \end{proof}

%% file: sections/appendix/appendix_lanczos.tex
\section{The Lanczos Gaussian Sampler}\label{app:lanczos}

In this section we provide a detailed description of the Lanczos Gaussian sampler along with complete proofs from Section~\ref{sec:lanczos}. Algorithm~\ref{alg:lanczos-sqrt-full} gives a complete description of the Lanczos approximation to $A^{1/2}v$. We note that the Lanczos iteration can be written in multiple ways. Our choice of implementation, Algorithm~\ref{alg:lanczos-sqrt-full}, is the one analyzed in \citet{musco2018stability} which has been shown to be the most numerically stable variant of the Lanczos iteration \citep{cullum2002lanczos, saad2011numerical}. 

We observe that the final step of Algorithm~\ref{alg:lanczos-sqrt-full} requires computing $T_m^{1/2}e_1$. Since $T_m$ is a small $m\times m$ symmetric tridiagonal matrix, $T_m^{1/2}e_1$ can be computed efficiently by diagonalizing $T_m$.

\begin{algorithm}[h]
\caption{Lanczos approximation to $A^{1/2}v$}
\label{alg:lanczos-sqrt-full}
\begin{algorithmic}[1]
\Require PSD matrix-vector product oracle $Ax$, vector $v\in\mathbb R^d$, number of Lanczos steps $m$
\Ensure Approximation $y_m \approx A^{1/2}v$
\If{$\|v\|_2 = 0$}
    \State \Return $0$
\EndIf
\State $q_0 \gets 0$, $\beta_1 \gets 0$, $q_1 \gets v/\|v\|_2$
\For{$j=1,\ldots,m$}
    \State $w \gets Aq_j - \beta_j q_{j-1}$
    \State $\alpha_j \gets q_j^\top w$
    \State $w \gets w - \alpha_j q_j$
    \Comment{Optional: reorthogonalize $w$ against $q_1,\ldots,q_j$ in finite precision}
    \State $\beta_{j+1} \gets \|w\|_2$
    \If{$\beta_{j+1}=0$}
        \State $m \gets j$
        \State \textbf{break}
    \EndIf
    \State $q_{j+1} \gets w/\beta_{j+1}$
\EndFor
\State $Q_m \gets [q_1,\ldots,q_m]$
\State Form the symmetric tridiagonal matrix
\[
T_m \gets 
\begin{bmatrix}
\alpha_1 & \beta_2 \\
\beta_2 & \alpha_2 & \ddots \\
& \ddots & \ddots & \beta_m \\
& & \beta_m & \alpha_m
\end{bmatrix}.
\]
\State \Return $y_m \gets \|v\|_2 Q_m T_m^{1/2} e_1$
\end{algorithmic}
\end{algorithm}

%% file: sections/appendix/appendix_main_result.tex
\section{Omitted Proofs from Section~\ref{sec:approximation_error}}
\label{app:main-proofs}

In this section of the appendix we provide complete proofs that are omitted from Section~\ref{sec:approximation_error}. 

\subsection{Proof of Proposition~\ref{prop:unconditional-path-kl}}

\begin{proposition}[Restatement of Proposition~\ref{prop:unconditional-path-kl}] The unconditional path KL satisfies 
\[
\KL\!\left(q(x_{1:T}) \,\middle\|\, p_\theta(x_{1:T})\right)
=
\mathcal{L}(\theta) 
- \mathcal{L}^*
+ \KL\!\left(q(x_T) \,\middle\|\, \N(0,I) \right)
\]
\end{proposition} 

\begin{proof}
By construction of the learned DDPM reverse kernel \citep{ho2020denoising}, 
\[
p_\theta(x_{1:T}) = p(x_T) \prod_{t=2}^T p_\theta(x_{t-1}\mid x_t) = \N(0,I) \prod_{t=2}^T p_\theta(x_{t-1}\mid x_t)
\]
As for $q(x_{1:T})$, the chain rule together with the Markov property of the forward process yields that the time-reversed process is Markov. Hence  
\[
q(x_{1:T}) = q(x_T) \prod_{t=2}^T q(x_{t-1}\mid x_t)
\]
We then observe that
\begin{align*}
  \KL\!\left(q(x_{1:T}) \,\middle\|\, p_\theta(x_{1:T})\right) 
  & = 
  \E_{q(x_{1:T})}\!\left[
    \sum_{t=2}^T \log\frac{q(x_{t-1}\mid x_t)}{p_\theta(x_{t-1}\mid x_t)} + 
    \log\frac{q(x_T)}{\N(0,I)} 
  \right] \\
  & =
  \sum_{t=2}^T \E_{q(x_t,x_{t-1})}\!\left[\log\frac{q(x_{t-1}\mid x_t)}{p_{\theta}(x_{t-1}\mid x_t)}\right]
  + \E_{q(x_T)}\!\left[\log\frac{q(x_T)}{\N(0,I)}\right] \\
  & = 
  \sum_{t=2}^T \E_{q(x_t,x_{t-1})}\!\left[\log\frac{q(x_{t-1}\mid x_t)}{p_{\theta}(x_{t-1}\mid x_t)}\right] 
  + \KL\!\left(q(x_T) \,\middle\|\, \N(0,I) \right)
\end{align*}

Where the second equality follows by the fact that each summand depends only on $x_t$ and $x_{t-1}$ (and only $x_T$ for the terminal term). For an arbitrary $t \in \{2,\dots,T\}$, we have that 

\begin{align*}
\mathcal{L}_t(\theta) - \mathcal{L}_t^*
& =
\E_{q(x_t,x_0)}\!\left[\KL\!\big(q(x_{t-1}\mid x_t,x_0)\,\big\|\,p_{\theta}(x_{t-1}\mid x_t)\big)\right] \\
& \quad
- \E_{q(x_t,x_0)}\!\left[\KL\!\big(q(x_{t-1}\mid x_t,x_0)\,\big\|\,q(x_{t-1}\mid x_t)\big)\right] \\
& = 
\E_{q(x_t,x_0)}\left[\E_{q(x_{t-1}\mid x_t,x_0)}\!\left[\log\frac{q(x_{t-1}\mid x_t)}{p_{\theta}(x_{t-1}\mid x_t)}\right]\right] \\
& = 
\E_{q(x_t,x_{t-1},x_0)}\!\left[\log\frac{q(x_{t-1}\mid x_t)}{p_{\theta}(x_{t-1}\mid x_t)}\right]  \\
& = 
\E_{q(x_t,x_{t-1})}\!\left[\log\frac{q(x_{t-1}\mid x_t)}{p_{\theta}(x_{t-1}\mid x_t)}\right]
\end{align*}

where the last equality follows by integrating out $x_0$. Substituting the above identity into the expression for the path KL, we obtain 
\begin{align*}
\KL\!\left(q(x_{1:T}) \,\middle\|\, p_\theta(x_{1:T})\right)
& = \sum_{t=2}^T \left(\mathcal{L}_t(\theta) - \mathcal{L}_t^*\right) + \KL\!\left(q(x_T) \,\middle\|\, \N(0,I) \right) \\
& = \mathcal{L}(\theta) - \mathcal{L}^* + \KL\!\left(q(x_T) \,\middle\|\, \N(0,I) \right),
\end{align*}        
the desired result as claimed.
\end{proof}

\subsection{Proof of Lemma~\ref{le:gaussian_taylor}}
\label{app:proof-gaussian-taylor}

\begin{lemma}[Restatement of Lemma~\ref{le:gaussian_taylor}]\label{lemma:gaussian-taylor-restated}
Assume $\norm{x_0}\le D$ almost surely. Then
\[
  \mathcal{L}_t(\theta^*)
  \;\le\;
  \frac{\gamma_t}{2}\,\E_{q(x_t)}\!\left[\tr\Cov(x_0\mid x_t)\right]
  - \frac{\gamma_t^2}{4}\,\E_{q(x_t)}\!\left[\tr\Cov(x_0\mid x_t)^2\right]
  + \frac{\gamma_t^3 D^6}{6}.
\]
\end{lemma}

\begin{proof}
Throughout, write $\Sigma_t := \Cov(x_0\mid x_t)$. The argument has two parts: (i) a pointwise computation showing $\mathcal{L}_t(\theta^*) = \tfrac{1}{2}\,\E_{q(x_t)}[\log\det(I+\gamma_t \Sigma_t)]$, and (ii) the scalar truncation $\log(1+u)\le u - u^2/2 + u^3/3$ for $u\ge 0$.

\paragraph{Step 1: pointwise minimization.}
For a fixed $x_t$, define
\[
  J(x_t)
  \;:=\;
  \min_{\mu_p\in\R^d,\,\Sigma_p\succ 0}\,
  \E_{q(x_0\mid x_t)}\!\left[\KL\!\big(q(x_{t-1}\mid x_t,x_0)\,\big\|\,\N(\mu_p,\Sigma_p)\big)\right],
\]
so that $\mathcal{L}_t(\theta^*) = \E_{q(x_t)}[J(x_t)]$. It suffices to evaluate $J(x_t)$ for each $x_t$. Write $q := q(x_{t-1}\mid x_t,x_0)$ with mean $\mu_q$ and covariance $\Sigma_q$, and $p := \N(\mu_p,\Sigma_p)$. The Gaussian-vs-Gaussian KL formula gives
\[
  \KL(q\,\|\,p)
  =
  \tfrac{1}{2}\!\left(
    \tr(\Sigma_p^{-1}\Sigma_q)
    + (\mu_p-\mu_q)^\top \Sigma_p^{-1}(\mu_p-\mu_q)
    - d
    + \log\tfrac{\det\Sigma_p}{\det\Sigma_q}
  \right).
\]
Let $m := \E_{q(x_0\mid x_t)}[\mu_q]$. Adding and subtracting $m$ inside each factor and using that $\mu_p$ and $m$ are constant under $\E_{q(x_0\mid x_t)}$,
\begin{equation}\label{eq:mu-bias-variance}
  \E_{q(x_0\mid x_t)}\!\left[(\mu_p-\mu_q)(\mu_p-\mu_q)^\top\right]
  = (\mu_p-m)(\mu_p-m)^\top + \Cov_{q(x_0\mid x_t)}(\mu_q).
\end{equation}
Substituting \eqref{eq:mu-bias-variance} into the quadratic-form term,
\[
  \E_{q(x_0\mid x_t)}\!\left[(\mu_p-\mu_q)^\top \Sigma_p^{-1}(\mu_p-\mu_q)\right]
  = (\mu_p-m)^\top \Sigma_p^{-1}(\mu_p-m)
  + \tr\!\big(\Sigma_p^{-1}\Cov_{q(x_0\mid x_t)}(\mu_q)\big).
\]
Only the first summand depends on $\mu_p$, and (since $\Sigma_p^{-1}\succ 0$) it is a strictly convex quadratic minimized at $\mu_p^\star = m = \E_{q(x_0\mid x_t)}[\mu_q]$, attaining value $0$.

At $\mu_p = \mu_p^\star$, the quadratic-form term vanishes and the two $\Sigma_p$-coupled trace terms combine into $\tr(\Sigma_p^{-1}C)$, where
\[
  C \;:=\; \E_{q(x_0\mid x_t)}[\Sigma_q] + \Cov_{q(x_0\mid x_t)}(\mu_q) \;\succ\; 0.
\]
Hence
\[
  J(x_t)
  \;=\;
  \min_{\Sigma_p \succ 0}\,
  \tfrac{1}{2}\!\left(
    \tr(\Sigma_p^{-1}C) - d + \log\det\Sigma_p - \E_{q(x_0\mid x_t)}[\log\det\Sigma_q]
  \right).
\]
Set $A := \Sigma_p^{-1/2} C \Sigma_p^{-1/2}\succ 0$, so that $\tr(\Sigma_p^{-1}C) + \log\det\Sigma_p = \tr(A) - \log\det A + \log\det C$. Since $\log\det C$ is constant in $\Sigma_p$, minimizing reduces to minimizing $\tr A - \log\det A = \sum_{i=1}^d (\nu_i - \log\nu_i)$ over the eigenvalues $\nu_1,\dots,\nu_d > 0$ of $A$. Each scalar map $\nu \mapsto \nu - \log\nu$ is strictly convex with unique minimum at $\nu = 1$, so $A^\star = I$, i.e.\ $\Sigma_p^\star = C$, and at this optimum $\tr((\Sigma_p^\star)^{-1}C) = d$, giving
\begin{equation}\label{eq:Lt-after-Sigma}
  J(x_t) \;=\; \tfrac{1}{2}\!\left(\log\det C - \E_{q(x_0\mid x_t)}[\log\det\Sigma_q]\right).
\end{equation}

It remains to evaluate the two terms in \eqref{eq:Lt-after-Sigma}. By the standard forward-posterior identity,
\[
  q(x_{t-1}\mid x_t,x_0) \;=\; \N\!\big(\tilde\mu_t(x_t,x_0),\, \tilde\beta_t I\big),
  \qquad
  \tilde\beta_t \;=\; \frac{(1-\bar\alpha_{t-1})\beta_t}{1-\bar\alpha_t},
\]
with $\tilde\mu_t$ affine in $x_0$ with coefficient $\sqrt{\bar\alpha_{t-1}}\beta_t/(1-\bar\alpha_t)$. Since $\Sigma_q = \tilde\beta_t I$ does not depend on $x_0$, $\E_{q(x_0\mid x_t)}[\log\det\Sigma_q] = d\log\tilde\beta_t$, and
\[
  C \;=\; \tilde\beta_t I + \frac{\bar\alpha_{t-1}\beta_t^2}{(1-\bar\alpha_t)^2}\,\Sigma_t.
\]
Substituting and using $\tilde\beta_t^{-1}\cdot\bar\alpha_{t-1}\beta_t^2/(1-\bar\alpha_t)^2 = \bar\alpha_{t-1}\beta_t/[(1-\bar\alpha_t)(1-\bar\alpha_{t-1})] = \gamma_t$,
\[
  J(x_t)
  \;=\; \tfrac{1}{2}\!\left(\log\det C - d\log\tilde\beta_t\right)
  \;=\; \tfrac{1}{2}\log\det\!\big(\tilde\beta_t^{-1}C\big)
  \;=\; \tfrac{1}{2}\log\det(I + \gamma_t\,\Sigma_t).
\]
Taking expectation over $x_t$,
\begin{equation}\label{eq:Lt-closed-form}
  \mathcal{L}_t(\theta^*) \;=\; \E_{q(x_t)}[J(x_t)] \;=\; \tfrac{1}{2}\,\E_{q(x_t)}\!\left[\log\det(I + \gamma_t\,\Sigma_t)\right].
\end{equation}

\paragraph{Step 2: scalar truncation and the cubic bound.}
The scalar inequality
\[
  \log(1+u) \;\le\; u - \tfrac{u^2}{2} + \tfrac{u^3}{3},
  \qquad u \ge 0,
\]
holds because $h(u) := u - u^2/2 + u^3/3 - \log(1+u)$ satisfies $h(0) = 0$ and $h'(u) = 1 - u + u^2 - 1/(1+u) = u^3/(1+u) \ge 0$ for $u \ge 0$.

Let $\lambda_1,\dots,\lambda_d \ge 0$ be the eigenvalues of $\Sigma_t \succeq 0$. Applying the inequality with $u = \gamma_t \lambda_j$ and summing,
\[
  \log\det(I + \gamma_t \Sigma_t)
  \;=\;
  \sum_{j=1}^d \log(1 + \gamma_t \lambda_j)
  \;\le\;
  \gamma_t\,\tr\Sigma_t - \tfrac{\gamma_t^2}{2}\,\tr\Sigma_t^2 + \tfrac{\gamma_t^3}{3}\,\tr\Sigma_t^3,
\]
where $\sum_j \lambda_j^k = \tr\Sigma_t^k$. Taking expectations and combining with \eqref{eq:Lt-closed-form},
\begin{equation}\label{eq:Lt-cubic-bound}
  \mathcal{L}_t(\theta^*)
  \;\le\;
  \frac{\gamma_t}{2}\,\E_{q(x_t)}\!\left[\tr\Sigma_t\right]
  - \frac{\gamma_t^2}{4}\,\E_{q(x_t)}\!\left[\tr\Sigma_t^2\right]
  + \frac{\gamma_t^3}{6}\,\E_{q(x_t)}\!\left[\tr\Sigma_t^3\right].
\end{equation}

Finally, the bounded-data assumption $\norm{x_0}\le D$ controls $\tr\Sigma_t^3$. By bias--variance decomposition,
\[
  \tr\Sigma_t \;=\; \E\!\left[\norm{x_0}^2 \,\big|\, x_t\right] - \norm{\E[x_0\mid x_t]}^2 \;\le\; D^2
\]
almost surely. Since $\Sigma_t\succeq 0$, $\sum_j \lambda_j^3 \le \big(\sum_j \lambda_j\big)^3$, so $\tr\Sigma_t^3 \le (\tr\Sigma_t)^3 \le D^6$ almost surely. Substituting into \eqref{eq:Lt-cubic-bound},
\[
  \mathcal{L}_t(\theta^*)
  \;\le\;
  \frac{\gamma_t}{2}\,\E_{q(x_t)}\!\left[\tr\Sigma_t\right]
  - \frac{\gamma_t^2}{4}\,\E_{q(x_t)}\!\left[\tr\Sigma_t^2\right]
  + \frac{\gamma_t^3 D^6}{6},
\]
which is the desired bound.
\end{proof}

\subsection{Proof of Lemma~\ref{le:exact_taylor}}

\begin{lemma}[Restatement of Lemma~\ref{le:exact_taylor}]
    Assuming that $\norm{x_0} \leq D$,
    \[
        \mathcal{L}^*_t
        \geq 
        \frac{\gamma_t}{2}\E_{q(x_t)}[\tr\Cov(x_0|x_t)] - \frac{\gamma_t^2}{4}\E_{q(x_t)}[\tr(\Cov(x_0|x_t)^2)] - \frac{\gamma_t^3D^6}{3} 
    \]
\end{lemma}

\begin{proof}

According to Proposition~\ref{prop:gaussian_channel_representation}, $\mathcal{L}^*_t = \E_{q(x_t)}[I(X_{0|t};\sqrt{\gamma_t} X_{0|t} + Z)]$ where $X_{0|t}\sim q(x_0|x_t)$ and $Z\sim \N(0,I)$ independent of $X_{0|t}$. We fix an arbitrary $x_t$ and bound $I(X_{0|t};\sqrt{\gamma} X_{0|t} + Z)$ pointwise, and then take expectation over $x_t$ to obtain the desired bound on $\mathcal{L}^*_t$. 

We define 
\[
g(\gamma) := I(X_{0|t};\sqrt{\gamma} X_{0|t} + Z) 
\]
We proceed by computing the second-order Taylor expansion of $g(\gamma)$ around $\gamma=0$. We immediately observe the $g(0)=0$ since $X_{0|t}$ and $Z$ are independent. Next, by the I-MMSE identity (Theorem 2 of \citet{guo2005mutual}),
\[
g'(\gamma) = \frac{1}{2}\mathrm{mmse}(X_{0|t}|\sqrt{\gamma} X_{0|t} + Z) = \frac{1}{2}\E\!\left[\|X_{0|t} - \E[X_{0|t}|\sqrt{\gamma} X_{0|t} + Z]\|^2\right]
\]
where the expectation is taken over the joint distribution of $X_{0|t}$ and $Z$. Since $X_{0|t}$ and $Z$ are independent, $\E[X_{0|t}|Z] = \E[X_{0|t}]$, so
\[
g'(0) = \frac{1}{2}\E\!\left[\|X_{0|t} - \E[X_{0|t}|Z]\|^2\right] =  \frac{1}{2}\E\!\left[\|X_{0|t} - \E[X_{0|t}]\|^2\right] = \frac{1}{2}\tr\Cov(X_{0|t})
\]
Next, by the second-derivative formula for Gaussian channels (Corollary 1 of \citet{payaro2009hessian}),
\[
g''(\gamma) = -\frac{1}{2}\E\!\left[\tr(\Cov(X_{0|t}|\sqrt{\gamma} X_{0|t} + Z)^2)\right]
\]
where the expectation is taken over the random vector $\sqrt{\gamma} X_{0|t} + Z$. Again, since $X_{0|t}$ and $Z$ are independent, $\Cov(X_{0|t}|Z) = \Cov(X_{0|t})$. Hence
\[
g''(0) = -\frac{1}{2}\E\!\left[\tr(\Cov(X_{0|t}|Z)^2)\right] = -\frac{1}{2}\tr(\Cov(X_{0|t})^2).
\]
As for the third derivative, we proceed by defining 
\[Y_\gamma := \sqrt{\gamma} X_{0|t} + Z, \qquad
\bar X_\gamma := X_{0|t} - \E[X_{0|t}|Y_\gamma], \qquad
T_\gamma := \E[\bar X_\gamma^{\otimes 3}|Y_\gamma].
\]
By the third-derivative formula for Gaussian channels (Theorem 1, in particular Example 3, of \citet{nguyen2024derivatives}),
\[
g'''(\gamma) = \E\!\left[\tr(\Cov(X_{0|t}|Y_\gamma)^3) - \frac{1}{2}\|T_\gamma\|_F^2\right]
\]
where the expectation is taken over the random vector $Y_\gamma$. Applying Taylor's theorem with remainder at $\gamma=0$, there exists $\zeta\in(0,\gamma_t]$ such that
\[
g(\gamma_t) = g(0) + g'(0)\gamma_t + \frac{g''(0)}{2}\gamma_t^2 + \frac{g'''(\zeta)}{6}\gamma_t^3.
\] Substituting the expressions above yields
\[
g(\gamma_t) = \frac{\gamma_t}{2}\tr\Cov(X_{0|t}) - \frac{\gamma_t^2}{4}\tr(\Cov(X_{0|t})^2)
+ \frac{\gamma_t^3}{6}\E\!\left[\tr(\Cov(X_{0|t}|Y_\zeta)^3) - \frac{1}{2}\|T_\zeta\|_F^2\right].
\]
Finally, for the desired lower bound we observe by Jensen's inequality that 
\[
\|T_\zeta\|_F = \norm{\E[\bar X_\zeta^{\otimes 3}|Y_\zeta]}_F \le \E\!\left[\norm{\bar X_\zeta^{\otimes 3}}_F \,\big|\, Y_\zeta\right] = \E\!\left[\norm{\bar X_\zeta}^3 \,\big|\, Y_\zeta\right]
\]
By $\norm{x_0}\le D$ almost surely,
\[
\norm{\bar X_\zeta} = \norm{X_{0|t} - \E[X_{0|t}|Y_\zeta]} \le \norm{X_{0|t}} + \norm{\E[X_{0|t}|Y_\zeta]} \le 2D
\]  
Therefore,
\begin{align*}
  \|T_\zeta\|_F 
  & \le 2D \E\!\left[\norm{\bar X_\zeta}^2 \,\big|\, Y_\zeta\right] \\
  & = 2D (\E\!\left[\norm{X_{0|t}}^2 \,\big|\, Y_\zeta\right] - \norm{\E[X_{0|t}|Y_\zeta]}^2) \\
  & \leq 2D^3
\end{align*}
Therefore, $\|T_\zeta\|_F^2 \le 4D^6$, and
\[
\E\!\left[\tr(\Cov(X_{0|t}|Y_\zeta)^3) - \frac{1}{2}\|T_\zeta\|_F^2\right] \ge -2D^6.
\]
Writing $\Cov(X_{0|t})$ as $\Cov(x_0|x_t)$ then substituting the above inequality into the expression for $I(X_{0|t};\sqrt{\gamma_t} X_{0|t} + Z)$ and taking expectation over $x_t$ gives the desired lower bound on $\mathcal{L}^*_t$.

\[
\mathcal{L}^*_t
\geq 
\frac{\gamma_t}{2}\E_{q(x_t)}[\tr\Cov(x_0|x_t)] - \frac{\gamma_t^2}{4}\E_{q(x_t)}[\tr(\Cov(x_0|x_t)^2)] - \frac{\gamma_t^3D^6}{3} 
\]
\end{proof}

%% file: sections/appendix/appendix_experiments.tex
\section{Experimental Details}
\label{app:experimental-details}

\subsection{Details of the Lanczos Gaussian sampler}
\label{app:experimental-details:lanczos-gaussian-sampler}

\textbf{Pretrained-networks for Hessian-vector products.} LGS only requires a pretrained score network, $s_\theta(x_t,t)$, which we obtain from the noise-prediction networks $\epsilon_\theta(x_t,t)$ released by prior work via
$s_\theta(x_t,t) = -\nabla_{x}\log p_\theta(x_t,t) = -\tfrac{1}{\sqrt{1-\bar{\alpha}_t}}\,\epsilon_\theta(x_t,t)$.
We use the checkpoints provided by \citet{bao2022analytic} for CIFAR-10, \citet{song2020denoising} for CelebA $64\times64$, and \citet{nichol2021improved} for ImageNet $64\times64$; these are the same checkpoints redistributed by \citet{ou2025ocm}, which makes our results directly comparable.

A Hessian--vector product (HVP) of $\log p_\theta(x_t,t)$ against a vector $v$ is computed by a single forward pass through $\epsilon_\theta$ followed by a single backward pass via PyTorch's \texttt{torch.autograd.grad} with \texttt{create\_graph=False} on the inner product $\langle \epsilon_\theta(x_t,t), v\rangle$.

\textbf{Ritz Value Clamping.} As noted in \citep{bao2022analytic} and \citep{ou2025ocm}, covariance clipping is essential to performance of diffusion models with unfixed variance in the backward kernel. An analogous clipping of the Ritz values (eigenvalues of $T_m$) is employed in the Lanczos Gaussian sampler: we clip below by $\tilde{\beta}_{t \to s}$ and above by $\tilde{\beta}_{t \to s} + \frac{\bar\alpha_s \beta_{t \to s}^2}{(1-\bar\alpha_t)^2}$. 
These bounds are motivated by the analytic form of the optimal covariance, under the standard assumption $0 \preceq \mathrm{Cov}[x_0\mid x_t] \preceq I$
(the same \texttt{clip\_cov\_x0} condition used in \citet{bao2022analytic}).

\textbf{Hutchinson diagonal guard.}
At the final stochastic reverse step, after drawing Lanczos noise $n\approx \Sigma_t^{1/2}z$, we stabilize
it by estimating $\mathrm{diag}(\Sigma_t)$ with $M$ Rademacher Hutchinson
probes, $\widehat{d}\approx M^{-1}\sum_j r_j\odot(\Sigma_t r_j)$, reusing the
same cached score graph so each probe costs one backward pass.
We clamp $\widehat{d}$ to the Ritz spectral interval, then rescale each
coordinate of $n$ by $\sqrt{\min(\widehat{d}_i,\sigma_{\mathrm{pix}}^2)/
\widehat{d}_i}$ with $\sigma_{\mathrm{pix}}$ fixed from \texttt{clip\_pixel}
\citep{bao2022analytic,ou2025ocm}.
This caps rare per-coordinate blow-ups while leaving typical coordinates
unchanged. We take the same hyperparameter as in \citep{ou2025ocm} for which denoising step index to apply this guard, and we use $M=5$ Hutchinson probes.

\textbf{LGS with batching and LGSb.}
We implement two variants of the Lanczos Gaussian sampler for speedup of wall-clock time for generating individual samples: (1) LGSb and (2) Window-LGSb. 

\textbf{LGS with batching.} Every $l$ denoising steps, we compute $l$ correlated noise vectors $\Sigma^{1/2} Z_1, \Sigma^{1/2} Z_2, \dots, \Sigma^{1/2} Z_l$ for independent Gaussian draws $Z_1, Z_2, \dots, Z_l$ and then utilize each of these at time steps $t, t-1, t-2, \dots, t-l+1$ respectively. 
The cost of doing this is that the network Jacobian used in the covariances used at time steps $t-1,t-2,\dots,t-l+1$ are effectively slightly 'stale', but this is not a concern if step sizes are sufficient small. 
The benefit of doing this is that these $l$ noise vectors are generated with a single tiled forward pass of
effective batch size $l \times B$, followed by $k$ shared backward passes
of the same width---one Lanczos iteration per backward. Thus, instead of accruing $k \times l$ backward passes through the score network if each noise vector was computed at its own denoising step, we only need to compute $k$ backward passes through the score network (assuming no overhead for parallelization). For our experiments, we use $l=2,3$.

\textbf{LGSb.} LGSb is the LGS method applied only to the last $w$ fraction of the denoising process, while the remaining denoising steps use the $\tilde{\beta}_t$ covariance. For our experiments, we use $w=0.25$.

For the impact of the above two methods on wall-clock time for generating individual samples, see Appendix~\ref{app:experimental-details:timing-analysis}.

\subsection{Image Experiments Setup}
\label{app:experimental-details:evaluation-details}
Our experimental methodology for evaluating the FID score closely follows that of \citep{ou2025ocm}.
The FID score is generated via 50K samples from the model. Following \citep{nichol2021improved, bao2022estimating, peebles2023scalable, ou2025ocm}, the reference distribution statistics are computed using the full training set for CIFAR-10 and ImageNet, and 50K training samples for CelebA 64x64.

For compute, we used single A100 80GB GPUs per 50K sample FID evaluations. We used batchsize of 200 for all the FID computations for all datasets and all methods. 

For Table~\ref{tab:fid_all_datasets}, we recompute the FID score also for DDPM $\beta$, DDPM $\tilde{\beta}$ and OCM-DDPM methods, while using the same sampling hyperparameters as prescribed by \citep{ou2025ocm}. For the LGS family of methods, we use the sampling hyperaparameters $k=3,5$ for the number of Lanczos iterations, and optionally $l=2,3$ for the number of denoising steps to batch. Note that $l=1$ is equivalent to the standard LGS method.

\subsection{Timing Analysis}
\label{app:experimental-details:timing-analysis}
To obtain a heuristic estimate of the wall-clock time for generating individual samples, we assume that the time for a single forward pass through the score network is the same as the time for a single backward pass. We also assume that the bottleneck for the sample generation time is the score network evaluations, and disregard other factors for this heuristic estimate.
Further, as claimed by \citet{ou2025ocm} and empirically validated by the authors and ourselves, the wall-clock time for single sample generation of OCM-DDPM is almost the same as that of DDPM $\beta$ and DDPM $\tilde{\beta}$, i.e. one forward pass through the score network.

For our LGS method, we require $1$ forward pass through the score network for each denoising step to predict the mean, and $k$ backward passes through the score network for each Lanczos iteration. Thus, we expect LGS to require $(k+1)$ times as much time as OCM-DDPM to generate a single sample.

For our LGS with batching method, every $l$ denoising steps, we compute $l$ noise vectors in parallel using the current timestep's score network JVPs. Assuming small enough $l$ so that the Lanczos algorithm for $l$ of these vectors may be run in parallel on the same GPU with no additional overhead, we expect LGS with batching to require $(k+1)$ network evaluations every $l$ denoising steps, and $1$ network evaluations for the next $l-1$ denoising steps (as we only need to compute the mean, and use the noise vector computed previously).
Thus, we expect LGS with batching to require $$(k+1) \times \frac{1}{l} + 1 \times \frac{l-1}{l} = \left(  1 + \frac{k}{l} \right)$$ times as much time as OCM-DDPM to generate a single sample.

For our LGSb method, we expect the LGS with batching runtime for only $w$ fraction of the denoising process, and the same runtime as OCM-DDPM for the remaining $1-w$ fraction of the denoising process. Thus, we expect LGSb to require $$\left(1 + \frac{k}{l}\right) \times w + (1) \times (1-w) = \left(1 + \frac{kw}{l}\right)$$ times as much time as OCM-DDPM to generate a single sample.

We empirically evaluate the wall-clock time for generating $n=99$ samples (post-warmup) for various methods and report the results in Table~\ref{tab:timing-analysis}. We report ratios of the median wall-clock time of each method to the median wall-clock time of OCM-DDPM, alongside the above heuristic estimate for each method.

\begin{table}[h]
\centering
\caption{%
  Wall-clock time per sample (seconds; $n=99$ post-warmup; $T=50$ steps).
  \emph{ratio} $=$ median $\div$ OCM-DDPM median for the same dataset.
  \emph{heuristic} is the predicted ratio derived in the text above (same for all datasets).%
}
\label{tab:timing-analysis}
\resizebox{\linewidth}{!}{%
\setlength{\tabcolsep}{4.5pt}%
\renewcommand{\arraystretch}{0.95}%
\begin{tabular}{@{}l rrrr rrrr rrrr r@{}}
\toprule
& \multicolumn{4}{c}{CIFAR-10 (LS)}
& \multicolumn{4}{c}{\textsc{CelebA}~$64\times64$}
& \multicolumn{4}{c}{\textsc{ImageNet}~$64\times64$}
& \\
\cmidrule(lr){2-5}\cmidrule(lr){6-9}\cmidrule(lr){10-13}
Model
& mean\,(s) & med.\,(s) & std\,(s) & ratio
& mean\,(s) & med.\,(s) & std\,(s) & ratio
& mean\,(s) & med.\,(s) & std\,(s) & ratio
& heuristic \\
\midrule
DDPM, $\beta$
  & 1.183 & 1.184 & 0.003 & \textbf{0.989}
  & 0.805 & 0.805 & 0.002 & \textbf{0.987}
  & 1.221 & 1.221 & 0.004 & \textbf{0.985}
  & \textbf{1.00} \\
DDPM, $\tilde{\beta}$
  & 1.223 & 1.209 & 0.050 & \textbf{1.010}
  & 0.811 & 0.811 & 0.003 & \textbf{0.994}
  & 1.232 & 1.232 & 0.003 & \textbf{0.994}
  & \textbf{1.00} \\
OCM-DDPM
  & 1.197 & 1.197 & 0.003 & \textbf{1.000}
  & 0.816 & 0.815 & 0.003 & \textbf{1.000}
  & 1.239 & 1.239 & 0.004 & \textbf{1.000}
  & \textbf{1.00} \\
\midrule
LGS ($k=3$)
  & 4.093 & 4.052 & 0.076 & \textbf{3.385}
  & 2.850 & 2.850 & 0.008 & \textbf{3.496}
  & 4.282 & 4.282 & 0.011 & \textbf{3.455}
  & \textbf{4.00} \\
LGS ($k=5$)
  & 5.836 & 5.768 & 0.097 & \textbf{4.819}
  & 4.111 & 4.103 & 0.027 & \textbf{5.032}
  & 6.181 & 6.180 & 0.018 & \textbf{4.987}
  & \textbf{6.00} \\
LGS ($k=3,\,l=2$)
  & 2.779 & 2.759 & 0.038 & \textbf{2.305}
  & 1.927 & 1.928 & 0.005 & \textbf{2.365}
  & 2.893 & 2.894 & 0.009 & \textbf{2.335}
  & \textbf{2.50} \\
LGS ($k=3,\,l=3$)
  & 2.281 & 2.281 & 0.010 & \textbf{1.906}
  & 1.585 & 1.584 & 0.006 & \textbf{1.942}
  & 2.415 & 2.416 & 0.011 & \textbf{1.950}
  & \textbf{2.00} \\
LGS ($k=5,\,l=2$)
  & 3.635 & 3.633 & 0.013 & \textbf{3.036}
  & 2.579 & 2.580 & 0.007 & \textbf{3.165}
  & 3.925 & 3.923 & 0.026 & \textbf{3.166}
  & \textbf{3.50} \\
LGS ($k=5,\,l=3$)
  & 2.931 & 2.906 & 0.089 & \textbf{2.428}
  & 2.037 & 2.036 & 0.007 & \textbf{2.497}
  & 3.039 & 3.037 & 0.010 & \textbf{2.451}
  & \textbf{2.67} \\
LGSb ($k=3,\,w=0.25$)
  & 1.994 & 1.967 & 0.074 & \textbf{1.644}
  & 1.365 & 1.365 & 0.004 & \textbf{1.674}
  & 2.043 & 2.044 & 0.005 & \textbf{1.649}
  & \textbf{1.75} \\
LGSb ($k=3,\,l=2,\,w=0.25$)
  & 1.639 & 1.639 & 0.019 & \textbf{1.369}
  & 1.119 & 1.117 & 0.007 & \textbf{1.370}
  & 1.695 & 1.696 & 0.005 & \textbf{1.368}
  & \textbf{1.38} \\
\bottomrule
\end{tabular}%
}
\end{table}